%% file: main.tex
\documentclass[10pt,journal,compsoc]{IEEEtran}
\usepackage[numbers]{natbib}
\usepackage{url}

\usepackage{ulem}
\usepackage{mathtools}
\usepackage{times}
\usepackage{algorithm, algpseudocode}
\usepackage{subcaption}
\usepackage{booktabs}
\usepackage[T1]{fontenc}

\usepackage{amsmath, bm}
\usepackage{amssymb}
\usepackage{amsthm}
\usepackage{tcolorbox}
\tcbuselibrary{skins,breakable}

\newtheorem{theorem}{Theorem}
\newtheorem{lemma}{Lemma}
\newtheorem{definition}{Definition}
\newtheorem{assumption}{Assumption}
\newcommand{\ind}{\perp\!\!\!\!\perp} 

\definecolor{DarkGreen}{RGB}{38,135,94}
\definecolor{LightGreen}{RGB}{238,247,243}
\newtcolorbox{greenbox}[1]{colback=LightGreen!35!white,
    colframe=DarkGreen!75!black,fonttitle=\bfseries,
    title={#1}}

\definecolor{DarkRed}{RGB}{184,31,0}
\definecolor{LightRed}{RGB}{218,118,91}
\newtcolorbox{redbox}[1]{colback=LightRed!15!white,
    colframe=DarkRed!75!black,fonttitle=\bfseries,
    title={#1}}

\begin{document}
\title{Meta-Learning with a Geometry-Adaptive Preconditioner}

\author{Suhyun~Kang, Duhun-Hwang, Moonjung~Eo, Taesup Kim and~Wonjong~Rhee,~\IEEEmembership{Fellow,~IEEE}
\IEEEcompsocitemizethanks{
\IEEEcompsocthanksitem This is an extended version of our previous CVPR'23 work in~\cite{kang2023meta}. We further validate the applicability and generalizability of GAP with two additional experiments -- 1. few-shot domain generalization and 2. reinforcement learning.
\IEEEcompsocthanksitem S.~Kang, D.~Hwang, and, W.~Rhee are with Deep Representation Learning Research Group, Department of Intelligence and Information, Seoul National University, Seoul, South Korea.\protect \\
E-mail: \{su\_hyun,~yelobean,~wrhee\}~@snu.ac.kr
\IEEEcompsocthanksitem M.~Eo is with Data Intelligence Lab, LG AI Research, Seoul, South Korea.\protect \\
E-mail: moonj@lgresearch.ai
\IEEEcompsocthanksitem T.~Kim is with Learning and Adaptation Algorithm Lab, Graduate School of Data Science, Seoul National University, Seoul, South Korea.\protect \\
Email: taesup.kim@snu.ac.kr
\IEEEcompsocthanksitem W.~Rhee is also with IPAI (Interdisciplinary Program in Artificial Intelligence) and AI Institute at Seoul National University.}
}

%
%

%

\IEEEtitleabstractindextext{%
\begin{abstract}
Model-agnostic meta-learning (MAML) is one of the most successful meta-learning algorithms. It has a bi-level optimization structure where the outer-loop process learns a shared initialization and the inner-loop process optimizes task-specific weights. Although MAML relies on the standard gradient descent in the inner-loop, recent studies have shown that controlling the inner-loop’s gradient descent with a meta-learned preconditioner can be beneficial. Existing preconditioners, however, cannot simultaneously adapt in a task-specific and path-dependent way. Additionally, they do not satisfy the Riemannian metric condition, which can enable the steepest descent learning with preconditioned gradient. In this study, we propose Geometry-Adaptive Preconditioned gradient descent (GAP) that can overcome the limitations in MAML; GAP can efficiently meta-learn a preconditioner that is dependent on task-specific parameters, and its preconditioner can be shown to be a Riemannian metric. Thanks to the two properties, the geometry-adaptive preconditioner is effective for improving the inner-loop optimization. 
Experiment results show that GAP outperforms the state-of-the-art MAML family and preconditioned gradient descent-MAML (PGD-MAML) family in a variety of few-shot learning tasks: few-shot regression, few-shot classification, cross-domain few-shot classification, few-shot domain generalization, and reinforcement learning.    
\end{abstract}

}

\maketitle

\IEEEdisplaynontitleabstractindextext

%
\IEEEpeerreviewmaketitle

%
\IEEEpeerreviewmaketitle

\input{0_Introduction}

\input{1_Backgrounds}
\input{2_Methodology}
\input{3_Proofs}
\input{4_Experiments}
\input{5_Discussion}

\input{6_Conclusion}
\input{7_Ack_bio}

\bibliographystyle{unsrtnat}
\bibliography{reference}

\end{document}

%% file: 0_Introduction.tex
\section{Introduction}
\label{sec:intro}
\IEEEPARstart{M}{eta-learning}, or \textit{learning to learn}, enables algorithms to quickly learn new concepts with only a small number of samples by extracting prior-knowledge known as meta-knowledge from a variety of tasks and by improving the generalization capability over the new tasks.

Among the meta-learning algorithms, the category of optimization-based meta-learning~\cite{finn2017model, finn2018probabilistic, raghu2019rapid, baik2021meta, ding2022gradient} has been gaining popularity due to its flexible applicability over diverse fields including robotics~\cite{song2020rapidly, wen2021multi}, medical image analysis~\cite{maicas2018training, singh2021metamed}, language modeling~\cite{mi2019meta, liu2020does}, and object detection~\cite{wu2020meta, perez2020incremental}.
In particular, Model-Agnostic Meta-Learning (MAML)~\cite{finn2017model} is one of the most prevalent gradient-based meta-learning algorithms.

\begin{figure*}[!t]
\centering
\subfloat[MAML]{\includegraphics[width=0.17\textwidth]{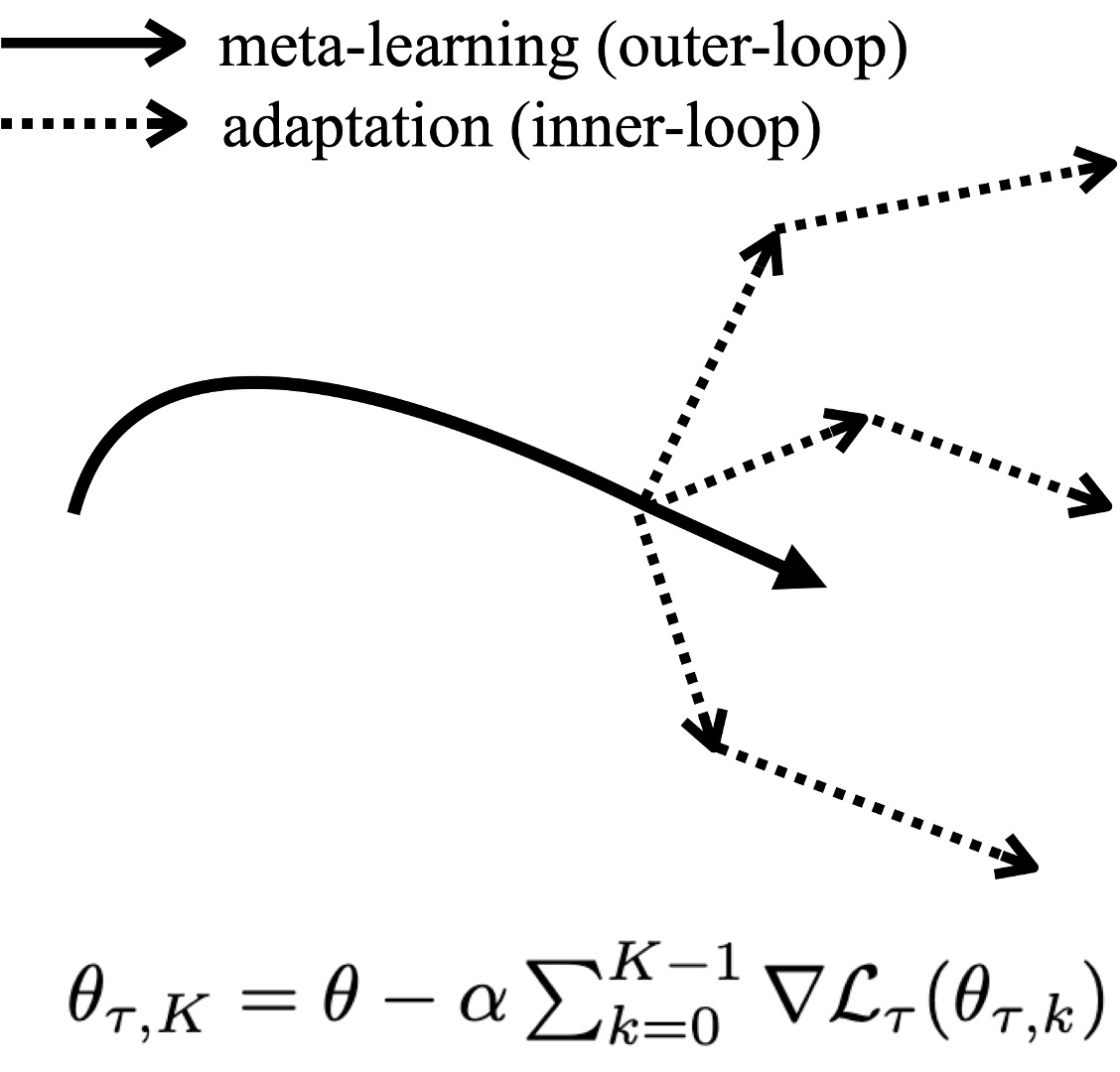}}
\subfloat[Non-adaptive $\mathbf{P}(\phi)$]{\includegraphics[width=0.20\textwidth]{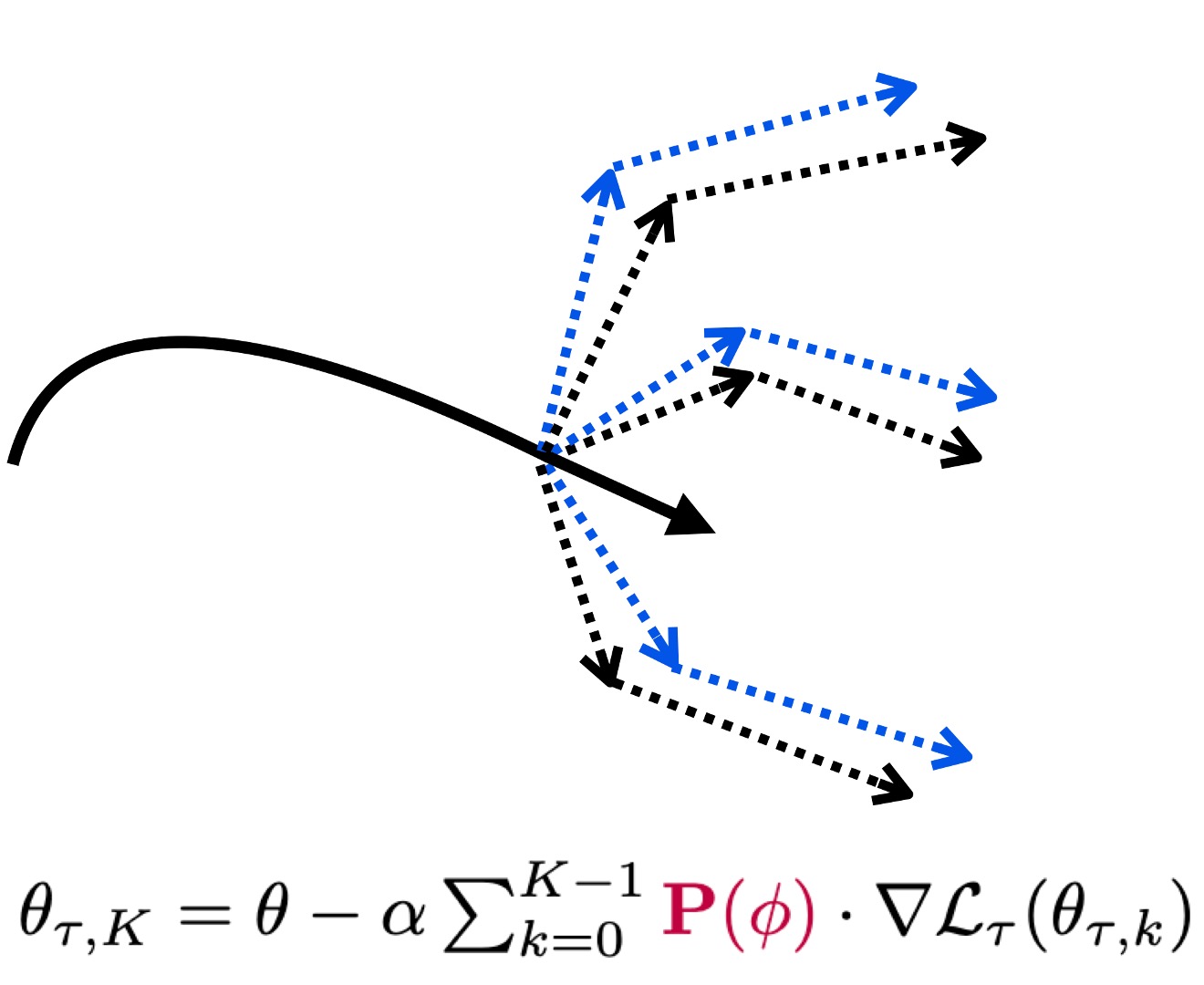}}
\subfloat[Adaptive $\mathbf{P}(k; \phi)$]{\includegraphics[width=0.20\textwidth]{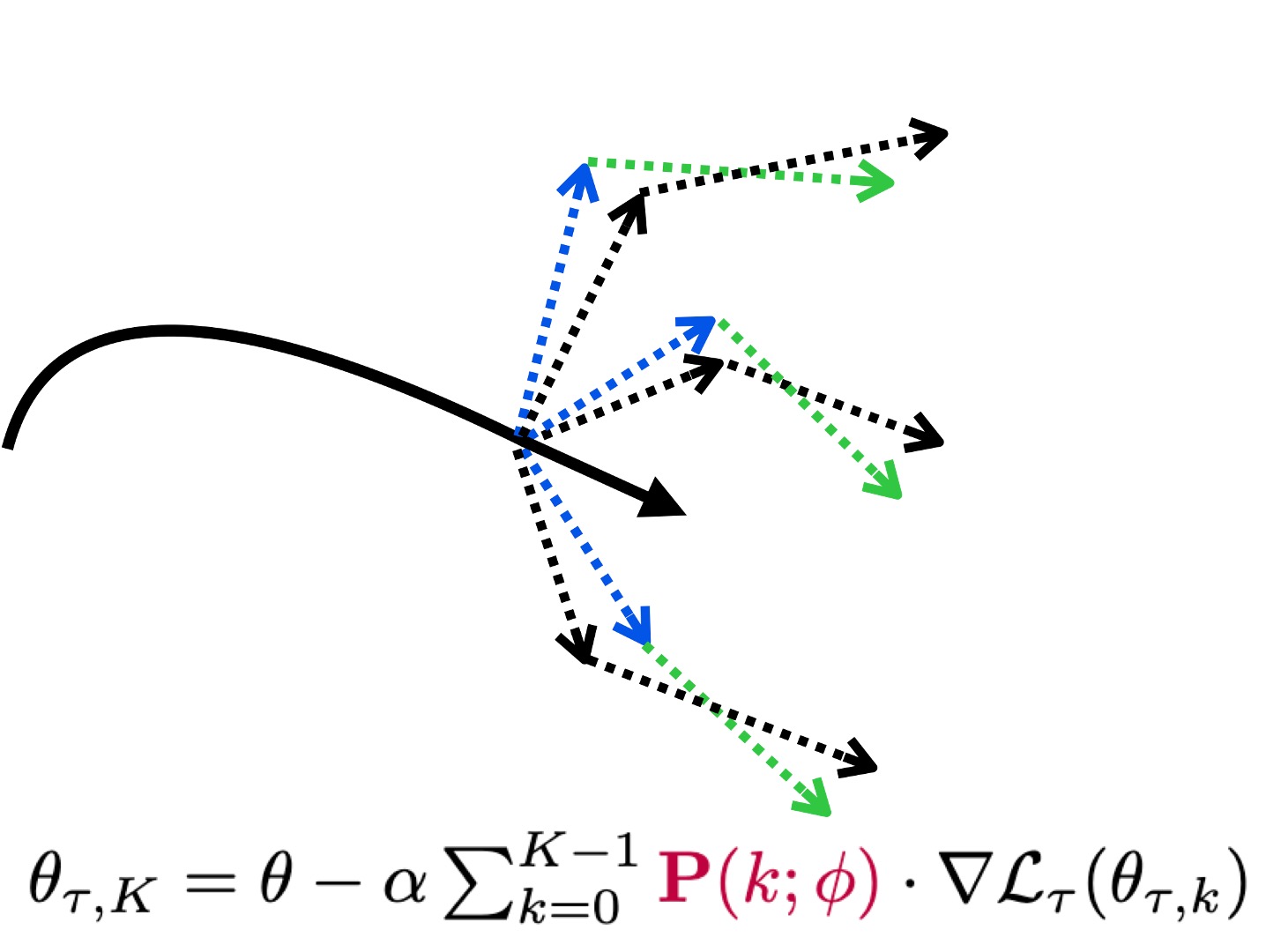}}
\subfloat[Adaptive $\mathbf{P}(D_{\tau}^{tr}; \phi)$]{\includegraphics[width=0.20\textwidth]{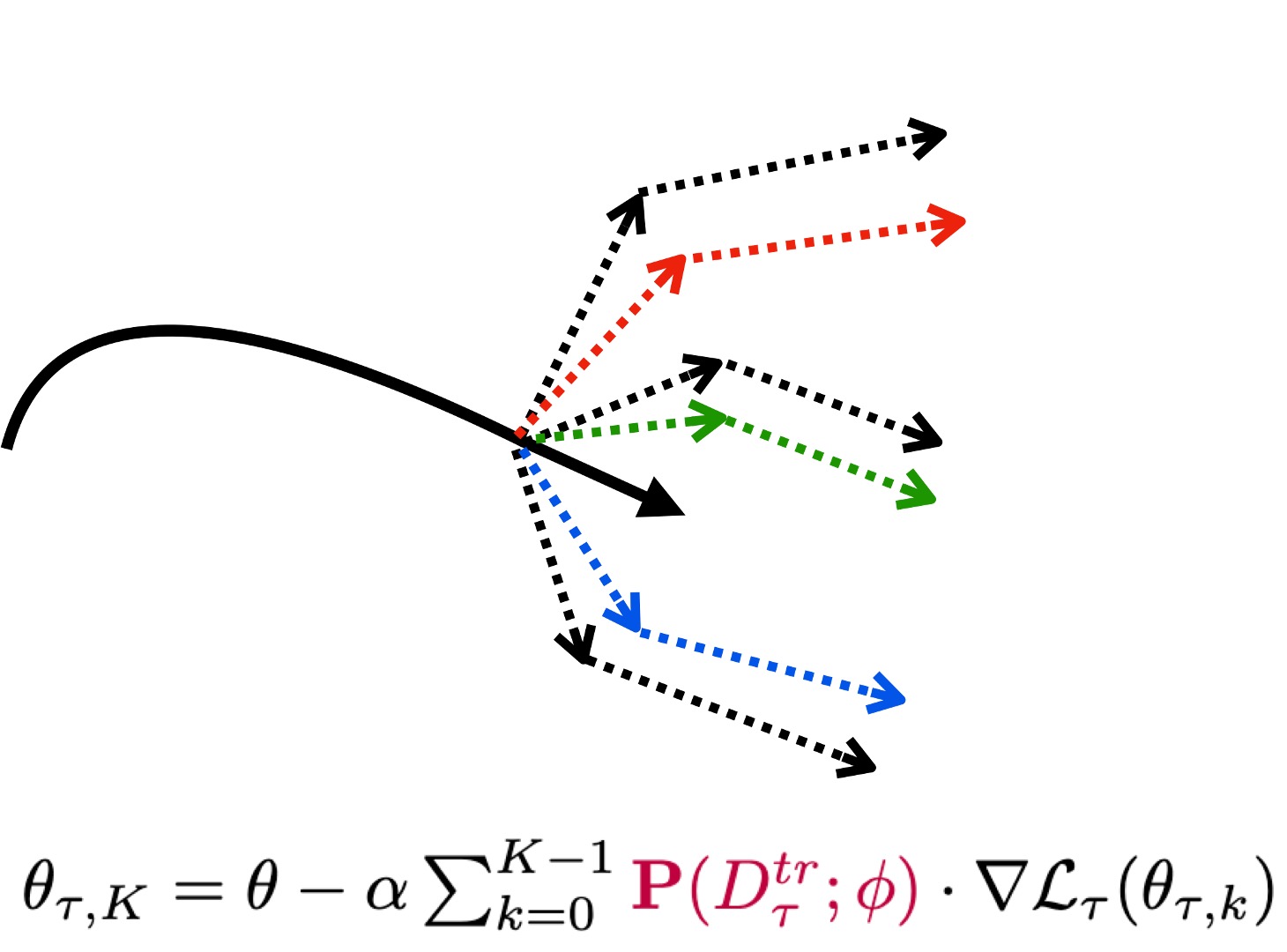}}
\subfloat[Adaptive $\mathbf{P}(\theta_{\tau, k}; \phi)$]{\includegraphics[width=0.20\textwidth]{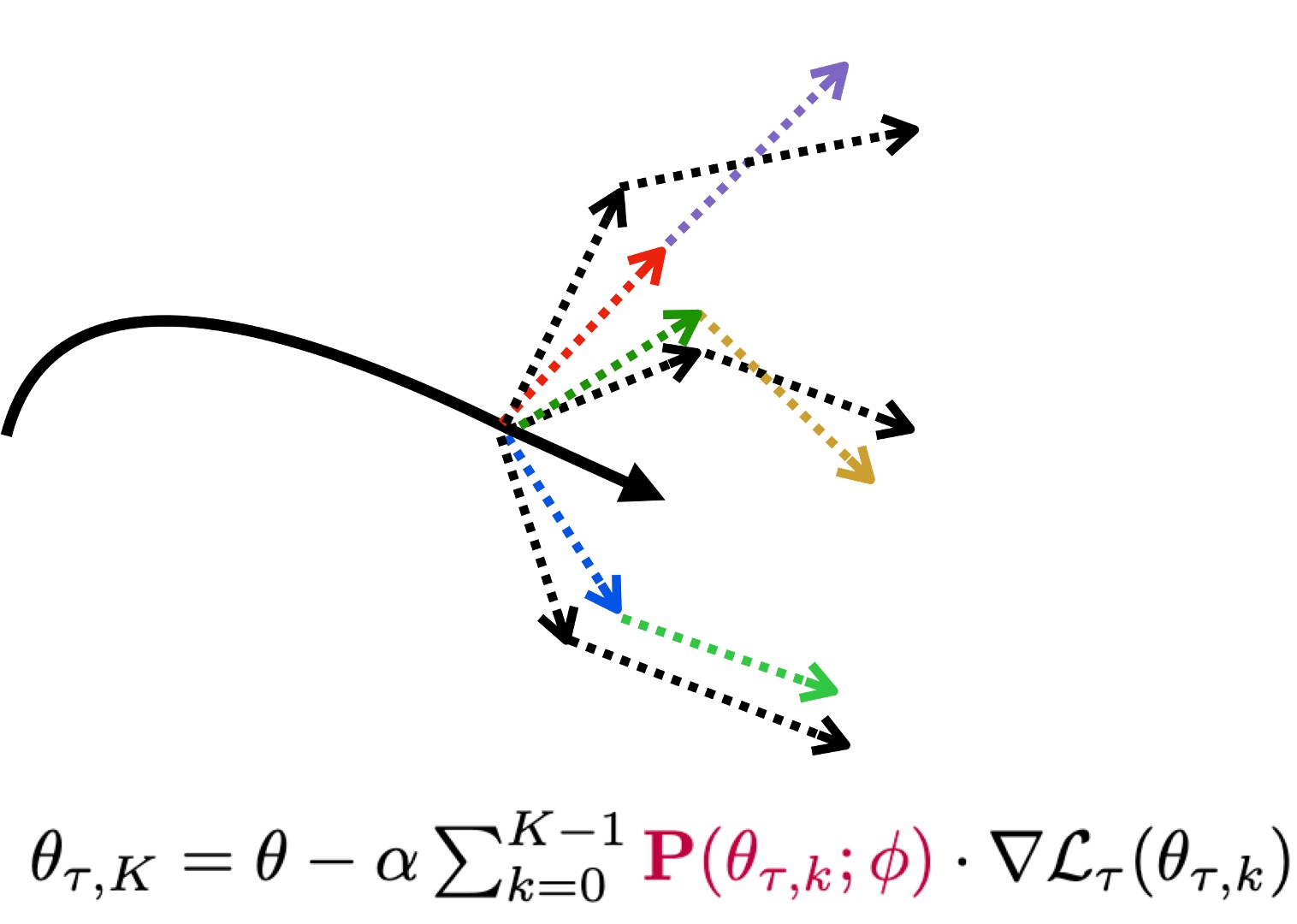}}
\caption{
Diagram of MAML and PGD-MAML family. For the inner-loop adaptation in each diagram, the dotted lines of the same color indicate that they use a common preconditioning matrix (preconditioner). (a) MAML adaptation: no preconditioner is used (i.e., $\mathbf{P}=\mathbf{I}$). (b) $\mathbf{P}(\phi)$: a constant preconditioner is used in the inner-loop where the preconditioner's meta-parameter $\phi$ is meta-learned. (c) $\mathbf{P}(k;\phi)$: a constant preconditioner is used for each inner-step $k$. Preconditioner for each step is meta-learned, but $\mathbf{P}(k,\phi)$ is not task-specific. (d) $\mathbf{P}(D_{\tau}^{\text{tr}};\phi)$: a constant preconditioner is used for each task. Preconditioner for each task is meta-learned, but $\mathbf{P}(D_{\tau}^{\text{tr}};\phi)$ is not dependent on $k$. (e) GAP adapts  $\mathbf{P}(\theta_{\tau, k}; \phi)$: a fully adaptive preconditioner is used where it is \textit{task-specific} and \textit{path-dependent}. Instead of saying `dependent on $k$', we specifically say it is \textit{path-dependent} because the exact dependency is on the task-specific parameter set $\theta_{\tau, k}$ that is considerably more informative than $k$.}
\label{fig:dependence_concept}
\end{figure*}

\begin{figure}[!t]
\centering
\captionsetup[subfloat]{justification=centering, labelfont=scriptsize, textfont=scriptsize}
\subfloat[Gradient Descent (GD)]{\includegraphics[width=0.25\textwidth]{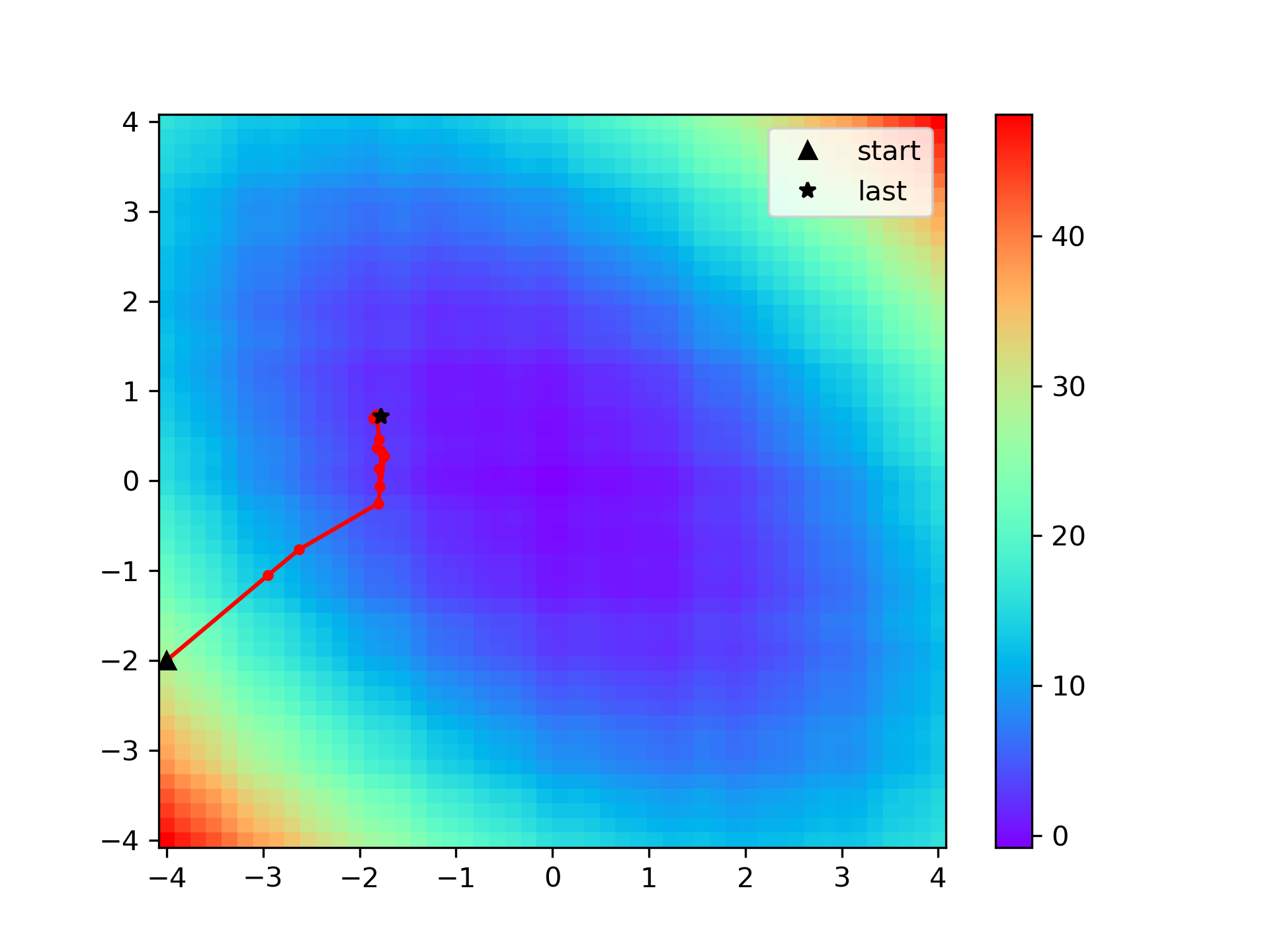}}%
\subfloat[Preconditioned GD\\ (Non-Riemannian metric)]{\includegraphics[width=0.25\textwidth]{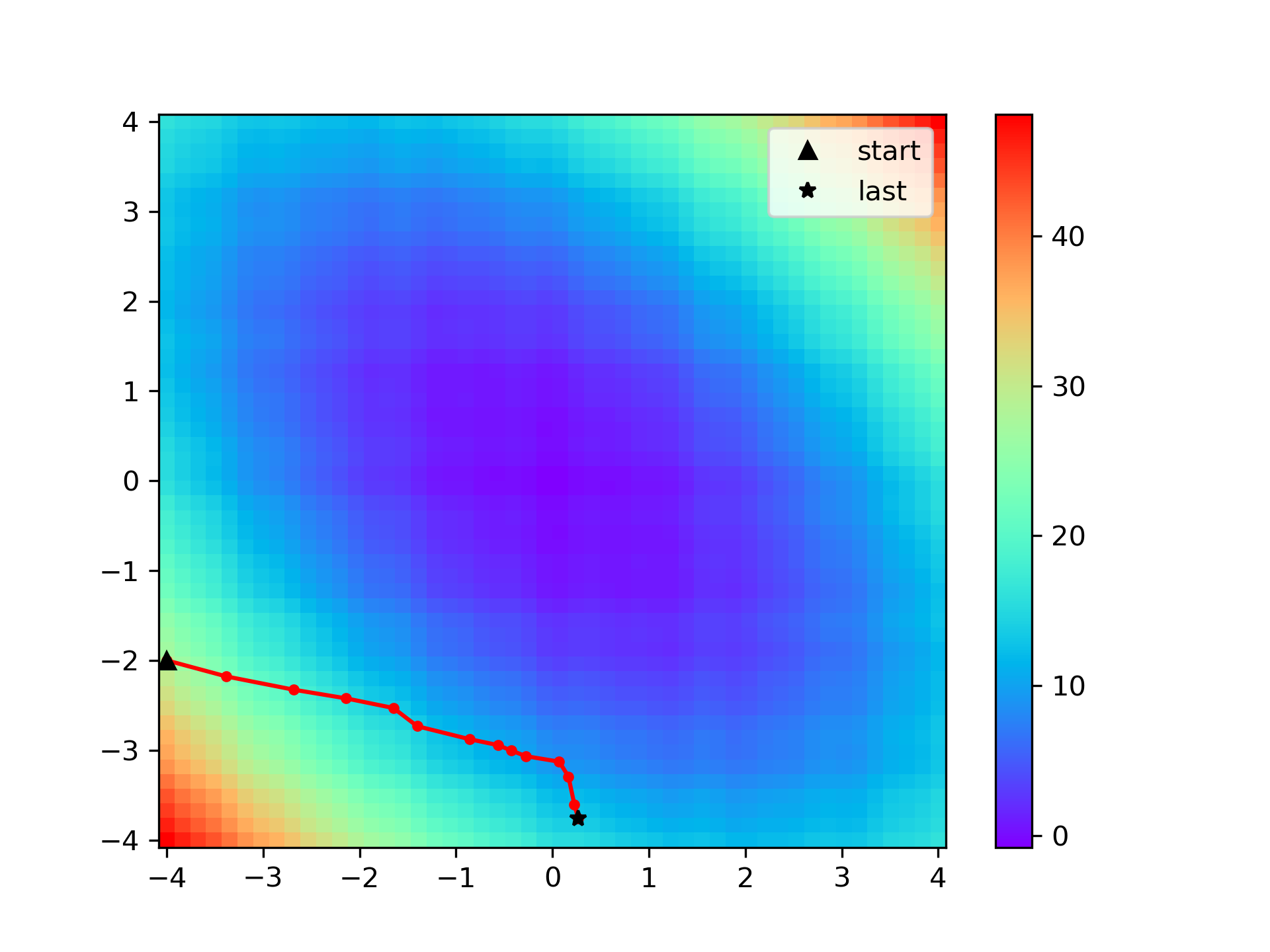}}\\
\subfloat[Preconditioned GD\\ (Riemannian metric)]{\includegraphics[width=0.25\textwidth]{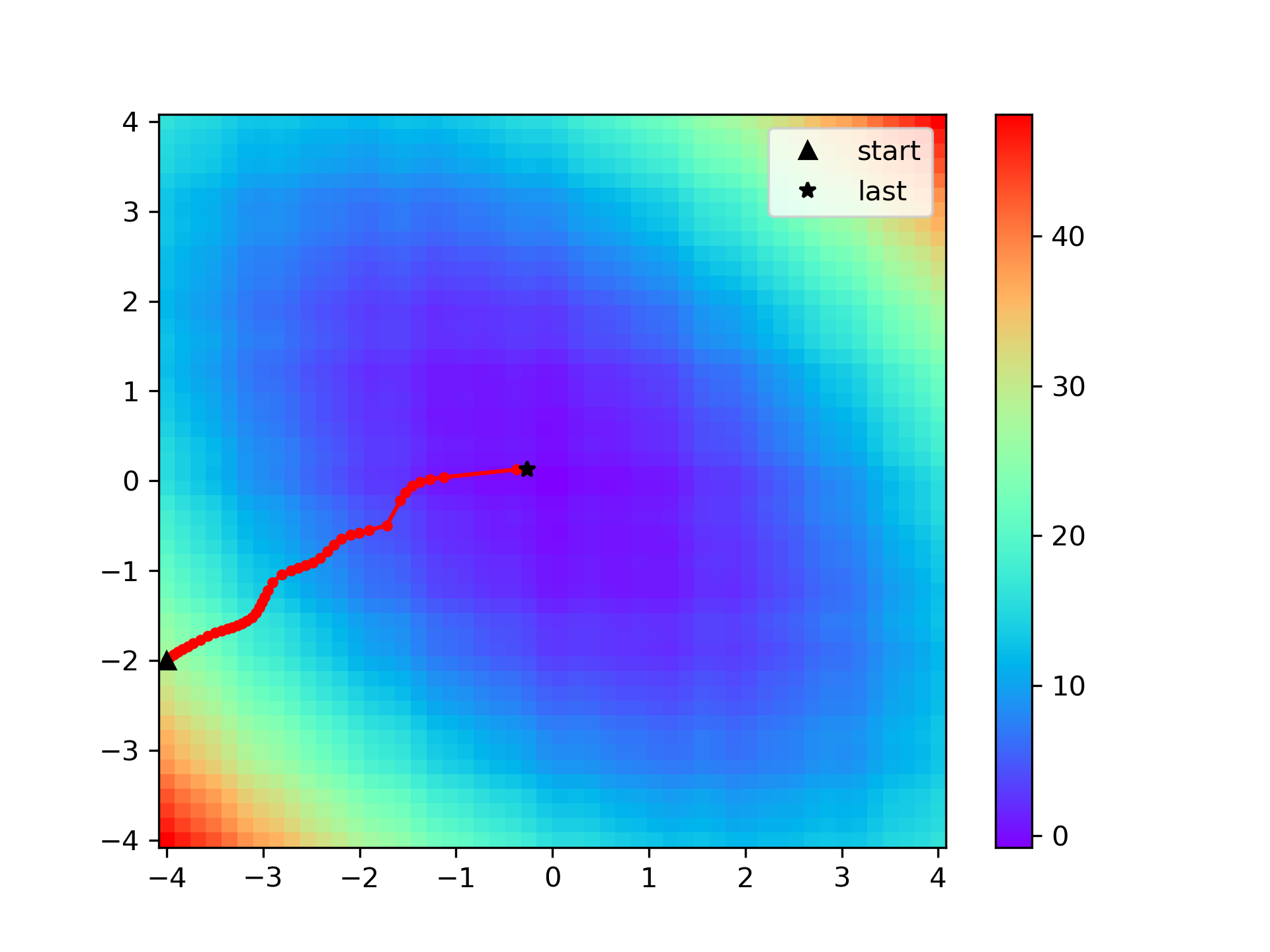}}%
\caption{%
A toy example for illustrating the effect of Riemannian metric. When the curvature of the parameter space is poorly conditioned, (a) gradient descent can suffer from the difficulty of finding the solution, (b) a preconditioned gradient descent with a preconditioner that does not satisfy the Riemannian metric condition can suffer further, and (c) a preconditioned gradient descent with a preconditioner that is a Riemannian metric can perform better.
}
\label{fig:riemannian_concept}
\end{figure}

Many recent studies have improved MAML by adopting a 
Preconditioned Gradient Descent~(PGD) for inner-loop optimization~\cite{li2017meta,lee2018gradient, park2019meta, simon2020modulating, rajasegaran2020meta, zhao2020meta, von2021learning}.
In this paper, we collectively address PGD-based MAML algorithms as the PGD-MAML family.
A PGD is different from the ordinary gradient descent because it performs a preconditioning on the gradient using a preconditioning matrix $\mathbf{P}$, also called a \textit{preconditioner}. 
A PGD-MAML algorithm meta-learns not only the initialization parameter $\theta_0$ of the network but also the meta-parameter $\phi$ of the preconditioner $\mathbf{P}$. 

For the inner-loop optimization, $\mathbf{P}$ was kept static in most of the previous works~(Fig.~\ref{fig:dependence_concept}(b))~\cite{li2017meta,lee2018gradient, park2019meta, zhao2020meta, von2021learning}. 
Some of the previous works considered adapting the preconditioner $\mathbf{P}$ with the inner-step $k$~(Fig.~\ref{fig:dependence_concept}(c))~\cite{rajasegaran2020meta} and some others with the individual task~(Fig.~\ref{fig:dependence_concept}(d))~\cite{simon2020modulating}. 
They achieved performance improvement by considering individual tasks and inner-step, respectively, and showed that both factors were valuable.
However, both factors have not been considered simultaneously in the existing studies.

When a parameter space has a certain underlying structure, there exists a Riemannian metric corresponding the parameter space~\cite{amari1996neural, amari1998natural}.
If the preconditioning matrix is the Riemannian metric, the preconditioned gradient is known to become the steepest descent on the parameter space~\cite{amari1967theory, amari1996neural, amari1998natural, amari1998natural2,kakade2001natural}.
An illustration of a toy example is shown in Fig.~\ref{fig:riemannian_concept}. 
The optimization path of an ordinary gradient descent is shown in Fig.~\ref{fig:riemannian_concept}(a).
Compared to the ordinary gradient descent, a preconditioned gradient descent with a preconditioner that does not satisfy the Riemannian metric condition can actually harm the optimization. For the example in Fig.~\ref{fig:riemannian_concept}(b), the preconditioner affects the optimization into an undesirable direction and negatively affects the gradient descent.
On the contrary, if the preconditioner is the Riemannian metric corresponding the parameter space, the preconditioned gradient descent can become the steepest descent and can exhibit a better optimization behavior as shown in Fig.~\ref{fig:riemannian_concept}(c).
While the Riemannian metric condition~(i.e, positive definiteness) is a necessary condition for steepest descent learning,
the existing studies on PGD-MAML family did not consider constraining preconditioners to satisfy the condition for Riemannian metric.

In this study, we propose a new PGD method named \textbf{G}eometry \textbf{A}aptive \textbf{P}reconditioned gradient descent (GAP). 
Specifically, GAP satisfies two desirable properties which have not been considered before.
First, GAP's preconditioner $\mathbf{P}_{\text{GAP}}$ is a fully adaptive preconditioner that can adapt to the individual task~(\textit{task-specific}) and to the optimization-path~(\textit{path-dependent}).
The full adaptation is made possible by having the preconditioner depend on the task-specific parameter $\theta_{\tau,k}$ (Fig.~\ref{fig:dependence_concept}(e)). 
Second, we prove that $\mathbf{P}_{\text{GAP}}$ is a Riemannian metric. 
To this end, we force the meta-parameters of $\mathbf{P}_{\text{GAP}}$ to be positive definite. 
Thus, GAP guarantees the steepest descent learning on the parameter space corresponding to $\mathbf{P}_{\text{GAP}}$.
Owing to the two properties, GAP enables a geometry-adaptive learning in inner-loop optimization.

For the implementation of GAP, we utilize the Singular Value Decomposition~(SVD) operation
to come up with our preconditioner satisfying the desired properties. 
For the recently proposed large-scale architectures, computational overhead can be an important design factor and we provide a low-computational approximation, \textit{Approximate GAP}, that can be proven to asymptotically approximate the operation of GAP.

To demonstrate the effectiveness of GAP, we empirically evaluate our algorithm on popular few-shot learning tasks; few-shot regression, few-shot classification, and few-shot cross-domain classification. 
The results show that GAP outperforms the state-of-the-art MAML family and PGD-MAML family.

%% file: 1_Backgrounds.tex
\section{Backgrounds}
\label{sec:backgrouds}
\subsection{Model-Agnostic Meta-Learning (MAML)}
\label{sec:2.1}
The goal of MAML~\cite{finn2017model} is to find the best initialization that the model can quickly adapt from, such that the model can perform well for a new task. MAML consists of two levels of main optimization processes: inner-loop and outer-loop optimizations. Consider the model $f_{\theta}(\cdot)$ with parameter $\theta$. For a task $\tau=\{D^{\text{tr}}_{\tau}, D^{\text{val}}_{\tau}\}$ sampled from the task distribution $p(\mathcal{T})$, the inner-loop optimization is defined as:
\begin{equation}
    \begin{split}
        \theta_{\tau, K} &=\theta_{\tau,0}-\alpha\sum_{k=0}^{K-1}\nabla_{\theta_{\tau,k}}\mathcal{L}^{\text{in}}_{\tau}(\theta_{\tau,k};D^{\text{tr}}_{\tau})\;\;s.t\;\;\theta_{\tau,0}=\theta,
    \end{split}
\end{equation}
where $\theta_{\tau, k}$ is task-specific parameters for task $\tau$, and $\alpha$ is the learning rate for inner-loop optimization, $\mathcal{L}_{\tau}^{\text{in}}$ is the inner-loop's loss function, and $K$ is the number of gradient descent steps.
With $D^{\text{val}}_{\tau_i}$ in each task, we can define outer-loop optimization as
\begin{equation}
    \theta\leftarrow\theta-\beta\nabla_{\theta}{\mathbb{E}}_{\tau}\left[\mathcal{L}_{\tau}^{\text{out}}(\theta_{\tau,K};D^{\text{val}}_{\tau})\right],
\end{equation}
where $\beta$ is the learning rate for outer-loop optimization, and $\mathcal{L}_{\tau}^{\text{out}}$ is the outer-loop's loss function.

\subsection{Preconditioned Gradient Descent (PGD)}
\label{sec:2.3}
PGD is a method that minimizes the empirical risk through a gradient update with a preconditioner that rescales the geometry of the parameter space. 
Given model parameters $\theta$ and task $\tau=\{D^{\text{tr}}_{\tau},D^{\text{val}}_{\tau}\}$, we can define the preconditioned gradient update with a preconditioner $\mathbf{P}$ as follows:
\begin{equation} \small
    \label{eqn:pgd}
    \begin{split}
        \theta_{\tau, k+1} &=\theta_{\tau,k}-\alpha\mathbf{P}\nabla_{\theta_{\tau,k}}\mathcal{L}_{\tau}(\theta_{\tau,k};D_{\tau}^{\text{tr}})\\
        & \;\;k=0,1,\cdots \text{ and } \theta_{\tau,0}=\theta
    \end{split},
\end{equation}
where $\mathcal{L}_{\tau}(\theta_{\tau,k};D_{\tau}^{\text{tr}})$ is the empirical loss for task $\tau$ and parameters $\theta_{\tau,k}$. 
Setting $\mathbf{P}=\mathbf{I}$ recovers Eq.~(\ref{eqn:pgd}) to the basic gradient descent~(GD). 
Choice of $\mathbf{P}$ for exploiting the second-order information includes the inverse Fisher information matrix $\mathbf{F}^{-1}$ which leads to the natural gradient descent (NGD)~\cite{amari1998natural}; 
the inverse Hessian $\mathbf{H}^{-1}$ which corresponds to the Newton's method~\cite{lecun2012efficient}; 
and the diagonal matrix estimation with the past gradients which results in the adaptive gradient methods~\cite{duchi2011adaptive, kingma2014adam}. 
They often reduce the effect of pathological curvature and speed up the optimization~\cite{amari2020does}. 

\subsection{Unfolding: reshaping a tensor into a matrix}
\label{sec:2.2}
In this study, the concept of \textit{unfolding} is used to transform the gradient tensor of convolutional kernels into a matrix form.
\textit{Tensor unfolding}, also known as matricization or flattening, is the process of reshaping the elements of an $N$-dimensional tensor $\textsf{\textbf{X}}\in\mathbb{R}^{I_1\times\cdots\times I_N}$ into a matrix~\cite{kolda2009tensor}.
The mode-$n$ unfolding of an $N$-dimensional tensor $\textsf{\textbf{X}}\in\mathbb{R}^{I_1\times\cdots\times I_N}$ is defined as: 
\begin{equation}
    \textsf{\textbf{X}}\xrightarrow[\text{mode-n unfolding}]{}
    \mathbf{X}_{[n]}\in\mathbb{R}^{I_n\times I_M}\text{, where } I_M=\prod_{k\neq n}I_k
\end{equation}
For example, the weight tensor of a convolutional layer is represented as a 4-D tensor ($\textsf{\textbf{W}}\in\mathbb{R}^{ C_{\text{out}} \times C_{\text{in}} \times k_h \times k_w}$), where it is composed of kernels and it can be unfolded into a matrix as one of the following four forms: 
(1) $\mathbf{W}_{[1]}\in \mathbb{R}^{C_{\text{out}}\times (C_{\text{in}}k_h k_w)}$, 
(2) $\mathbf{W}_{[2]}\in \mathbb{R}^{C_{\text{in}}\times (C_{\text{out}}k_h k_w)}$, 
(3) $\mathbf{W}_{[3]}\in \mathbb{R}^{k_h\times (C_{\text{out}}C_{\text{in}}k_w)}$, 
(4) $\mathbf{W}_{[4]}\in \mathbb{R}^{k_w\times (C_{\text{out}}C_{\text{in}}k_h)}$. 

\subsection{Riemannian manifold}
An $n$-dimensional Riemannian manifold is defined by a manifold $\mathcal{M}$ and a Riemannian metric $g:\mathcal{M}\rightarrow\mathbb{R}^{n\times n}$, which is a smooth function from each point $\bm{x}\in\mathcal{M}$ to a positive definite matrix~\cite{lee2012smooth}. 
The metric $g(\bm{x})$ defines the inner product of two tangent vectors for each point of the manifold $\langle\cdot,\cdot\rangle:\mathcal{T}_{\bm{x}}\mathcal{M}\times\mathcal{T}_{\bm{x}}\mathcal{M}\rightarrow\mathbb{R}$, where $\mathcal{T}_{\bm{x}}\mathcal{M}$ is the tangent space of $\bm{x}$. For $\bm{v},\bm{w} \in \mathcal{T}_{\bm{x}}\mathcal{M}$, the inner product can be expressed $\langle \bm{v}, \bm{w} \rangle=\bm{v}^T g(\bm{x}) \bm{w}$.
A Riemannian manifold can be characterized by the curvature of the curves defined by a metric. 
The curvature of a Riemannian manifold can be computed at each point of the curves, while some manifolds have curvatures of a constant value. For example, the unit sphere $\mathcal{S}$ has constant positive curvature of $+1$.

%% file: 2_Methodology.tex
\section{Methodology}
\label{methodology}
In this section, we propose a new preconditioned gradient descent method called GAP in the MAML framework.
In Section~\ref{sec:3.1}, we introduce GAP in the inner-loop optimization and describe how to meta-train GAP in the outer-loop optimization. 
In Section~\ref{sec:3.2}, we prove that GAP has two desirable properties.
In Section~\ref{sec:3.3}, we provide a low-computational approximation that can be useful for large-scale architectures.

\subsection{GAP: Geometry-Adaptive Preconditioner}
\label{sec:3.1}
\begin{algorithm}[t]
\scriptsize
    \caption{Geometry-Adaptive Preconditioned gradient descent (GAP)}\label{alg:metasvd}
    \begin{algorithmic}[1]
        \Require $p(\mathcal{T})$: distribution over tasks
        \Require $\alpha, \beta_1, \beta_2$: learning rates
        \State Randomly initialize $\theta=\{\textsf{\textbf{W}}^1,\cdots,\textsf{\textbf{W}}^L\}$. 
        \State Initialize $\phi=\{\mathbf{M}^1,\cdots,\mathbf{M}^L\}$ as identity matrix.
        \While{not converged}
            \State Sample a batch of tasks $\mathcal{T}_i \sim p(\mathcal{T})$
            \For{all $\tau\in\mathcal{T}_i$}
                \For{inner-loop step $k=0$ \textbf{to} $K-1$}
                    \For{layer $l=1$ \textbf{to} $L$}
                        \State Compute $\textsf{\textbf{G}}^l_{\tau,k}=\nabla_{\textsf{\textbf{W}}^l_{\tau,k}}\mathcal{L}^{\text{in}}_{\tau}(\theta_{\tau,k};D^{\text{tr}}_{\tau})$ using $D^{\text{tr}}_{\tau}$
                        \State Reshape $\textsf{\textbf{G}}^l_{\tau,k}$ to $\mathbf{G}^l_{\tau,k}$ via Eq.~(\ref{eqn:reshaping})
                        \State Transform $\mathbf{G}^l_{\tau,k}$ to $\mathbf{\tilde{G}}^l_{\tau,k}$ using $\mathbf{M}^l$ via Eq.~(\ref{eqn:precondition gradient})
                        \State Reshape $\mathbf{\tilde{G}}^l_{\tau,k}$ back to the original form of gradient tensor, $\tilde{\textsf{\textbf{G}}}^l_{\tau,k}$
                        \State Compute $l$-layer adapted weight:\ $\textsf{\textbf{W}}^l_{\tau,k+1}=\textsf{\textbf{W}}^l_{\tau,k}-\alpha\cdot\tilde{\textsf{\textbf{G}}}^l_{\tau,k}$
                    \EndFor
                \EndFor
                \State Compute $\mathcal{L}^{\text{out}}_{\tau}(\theta_{\tau, K};D^{\text{val}}_{\tau})$ by evaluating $\mathcal{L}^{\text{out}}_{\tau}$ w.r.t $D^{\text{val}}_{\tau}$.
            \EndFor
            \State Update the weights and meta parameters:
            \State $\theta \leftarrow \theta - \beta_{1}\nabla_{\theta}\sum_{\tau\in\mathcal{T}}\mathcal{L}_{\tau}^{\text{out}}(\theta_{\tau,K};D^{\text{val}}_{\tau})$
            \State $\phi \leftarrow \phi - \beta_2\nabla_{\phi}\sum_{\tau\in\mathcal{T}}\mathcal{L}_{\tau}^{\text{out}}(\theta_{\tau,K};D^{\text{val}}_{\tau})$
        \EndWhile
    \end{algorithmic}
\end{algorithm}

\subsubsection{Inner-loop optimization}
\label{sec:3.1.1}
We consider an $L$-layer neural network $f_{\theta}(\cdot)$ with parameters $\theta=\{\textsf{\textbf{W}}^1,\cdots,\textsf{\textbf{W}}^l,\cdots,\textsf{\textbf{W}}^L\}$. 
In the standard MAML with task $\tau\sim p(\mathcal{T})$, each $\textsf{\textbf{W}}^l$ is adapted with the gradient update as below:
\begin{equation}
    \textsf{\textbf{W}}^l_{\tau, K}\leftarrow \textsf{\textbf{W}}^l_{\tau,0} - \alpha\cdot\sum_{k=0}^{K-1}\textsf{\textbf{G}}^l_{\tau,k}\;\;s.t\;\;\textsf{\textbf{W}}^l_{\tau,0}=\textsf{\textbf{W}}^l,
\end{equation} 
where $\textsf{\textbf{G}}^l_{\tau,k}=\nabla_{\textsf{\textbf{W}}^l_{\tau,k}}\mathcal{L}^{\text{in}}_{\tau}(\theta_{\tau,k};D^{\text{tr}}_{\tau})$ is the gradient with respect to $\textsf{\textbf{W}}^l_{\tau,k}$ and $\alpha$ is the learning rate for inner-loop optimization.

In the GAP, we first use the mode-$1$ unfolding to reshape the gradient tensor into a matrix form~(see Section~\ref{sec:2.2}). 
For a convolutional layer (i.e., $\textsf{\textbf{G}}^l_{\tau,k} \in \mathbb{R}^{C_{\text{out}}\times C_{\text{in}} \times k \times k}$), we reshape the gradient tensor as below:
\begin{equation}
    \label{eqn:reshaping}
    \textsf{\textbf{G}}^l_{\tau,k}\xrightarrow[\substack{\text{mode-1}\\ \text{unfolding}}]{}
    \begin{cases}
        \mathbf{G}^l_{\tau,k}\in\mathbb{R}^{C_{\text{out}}\times C_{\text{in}}k^2} &  \text{if } C_{\text{out}} \leq C_{\text{in}}k^2\\
        \mathbf{G}^l_{\tau,k}\in\mathbb{R}^{C_{\text{in}}k^2\times C_{\text{out}}} &  \text{if } C_{\text{in}}k^2 < C_{\text{out}},\\
    \end{cases}
\end{equation} 
where $\mathbf{G}^l_{\tau,k}$ denotes $(\mathbf{G}^l_{\tau,k})_{[1]}$ for the notational brevity. 
Note that we chose mode-1 unfolding because it performs best among the four unfolding forms as shown in Table~\ref{tab:unfolding_performance}.
For a linear layer in a matrix form, there is no need for an unfolding.

Second, we transform the singular values of the gradient matrix using additional meta parameters $\phi=\{\mathbf{M}^l\}_{l=1}^{L}$. 
The meta parameter $\mathbf{M}^l$ are diagonal matrices with positive elements defined as:
\begin{equation}
    \mathbf{M}^l = 
    \begin{cases}
        \text{diag}(\text{Sp}(m^l_1),\cdots,\text{Sp}(m^l_{C_{\text{out}}})) & \text{if } C_{\text{out}} \leq C_{\text{in}}k^2\\
        \text{diag}(\text{Sp}(m^l_1),\cdots,\text{Sp}(m^l_{C_{\text{in}}k^2})) & \text{if } C_{\text{in}}k^2 < C_{\text{out}}\\
    \end{cases},
\end{equation}
where $m^l_i\in\mathbb{R}$ and $\text{Sp}(x)=\frac{1}{2}\cdot\log(1+\exp{(2*x)})$.
They are applied to the gradient matrix as follows:
\begin{equation}
\label{eqn:precondition gradient}
    \mathbf{\tilde{G}}^l_{\tau,k}=\mathbf{U}^l_{\tau,k}(\mathbf{M}^l\cdot \mathbf{\Sigma}^l_{\tau, k}){\mathbf{V}^l_{\tau,k}}^{\text{T}},
\end{equation} 
where $\mathbf{G}^l_{\tau,k}=\mathbf{U}^l_{\tau,k}\mathbf{\Sigma}^l_{\tau,k}{\mathbf{V}^l_{\tau,k}}^{\text{T}}$ is the singular value decomposition (SVD) of $\mathbf{G}^l_{\tau,k}$. 

Finally, we reshape $\mathbf{\tilde{G}}^l_{\tau,k}$ back to its original gradient tensor form $\tilde{\textsf{\textbf{G}}}^l_{\tau,k}$~(i.e., inverse unfolding).

The resulting preconditioned gradient descent of GAP becomes the following:
\begin{equation}
    \textsf{\textbf{W}}^l_{\tau, K}\leftarrow \textsf{\textbf{W}}^l_{\tau,0} - \alpha\cdot\sum_{k=0}^{K-1}\tilde{\textsf{\textbf{G}}}^l_{\tau,k}\;\;s.t\;\;\textsf{\textbf{W}}^l_{\tau,0}=\textsf{\textbf{W}}^l,
\end{equation}
where $\tilde{\textsf{\textbf{G}}}^l_{\tau,k}$ is the preconditioned gradient based on the meta parameters $\phi$.

\begin{table}[t]
    \caption{Performance comparison of unfolding types. We performed the experiment with 5-way 1-shot on mini-ImageNet and used the standard Conv-4 backbone.}
    \centering
    \resizebox{0.5\textwidth}{!}{
    \begin{tabular}{cccc}
        \hline
        mode-1  & mode-2 & mode-3 & mode-4 \\
        \hline
        $54.86 \pm 0.85$ & $53.02 \pm 0.87$ & $51.23 \pm 0.76$ & $51.45 \pm 0.77$\\
        \hline
    \end{tabular}
    }
    \label{tab:unfolding_performance}
\end{table}

\subsubsection{Outer-loop optimization}
\label{sec:3.1.2}
For outer-loop optimization, GAP follows the typical process of MAML.
Unlike MAML, however, GAP meta-learns two meta parameter sets $\theta$ and $\phi$ as follows:
\begin{align}
    & \theta \leftarrow \theta - \beta_{1}\nabla_{\theta}{\mathbb{E}}_{\tau}\left[\mathcal{L}_{\tau}^{\text{out}}(\theta_{\tau,K};D^{\text{val}}_{\tau})\right],\\
    & \phi \leftarrow \phi - \beta_2\nabla_{\phi}{\mathbb{E}}_{\tau}\left[\mathcal{L}_{\tau}^{\text{out}}(\theta_{\tau,K};D^{\text{val}}_{\tau})\right],
\end{align}
where $\beta_1$ and $\beta_2$ are the learning rates for the outer-loop optimization.
We initialize $\mathbf{M}^l$ as an identity matrix for all $l$. 
The training procedure is provided in Algorithm~\ref{alg:metasvd}.

\subsection{Desirable properties of GAP}
\label{sec:3.2}
In this section, we prove that GAP's preconditioner $\mathbf{P}_{\text{GAP}}$ satisfies two desirable properties.
\begin{theorem}
    \label{thm:sec.3.2}
    Let $\mathbf{\tilde{G}}^l_{\tau,k}\in\mathbb{R}^{m\times n}$ be the `$l$-layer $k$-th inner-step' gradient matrix transformed by meta parameter $\mathbf{M}^l$ for task $\tau$.
    Then preconditioner $\mathbf{P}_{\text{GAP}}$ induced by $\mathbf{\tilde{G}}^l_{\tau,k}$ is a Riemannian metric and depends on the task-specific parameters $\theta_{\tau,k}$.
\end{theorem}

The proof and the closed form of $\mathbf{P}_{\text{GAP}}$ are provided in Section~\ref{subsec:proof_of_thm1}.
The following two properties emerge from the theorem.
\newline
\textbf{Property 1. Dependency on task-specific parameters:}
Theorem~\ref{thm:sec.3.2} formally shows that $\mathbf{P}_{\text{GAP}}$ depends on the task-specific parameters $\theta_{\tau,k}$. While the previous studies considered a non-adaptive preconditioner $\mathbf{P}(\phi)$~\cite{li2017meta, lee2018gradient, park2019meta, von2021learning, zhao2020meta, rajasegaran2020meta} and a partially adaptive preconditioner $\mathbf{P}(k;\phi)$~\cite{rajasegaran2020meta} or $\mathbf{P}(D_{\tau}^{\text{tr}};\phi)$~\cite{simon2020modulating}, $\mathbf{P}_{\text{GAP}}$ can be considered to be the most advanced adaptive preconditioner because it is fully adaptive (i.e., task-specific and path-dependent) by being dependent on $\theta_{\tau,k}$ as shown in Fig.~\ref{fig:dependence_concept}.
\newline
\textbf{Property 2. Riemannian metric:}
If the parameter space has a certain underlying structure, the ordinary gradient of a function $\nabla\mathcal{L}$ does not represent its steepest direction~\cite{amari1998natural}.
To define the steepest direction on the parameter space, we need a Riemannian metric $g(\boldsymbol{w})$, which is a positive-definite matrix defined for each parameter $\boldsymbol{w}$.
A Riemannian metric defines the steepest descent direction by $-g(\boldsymbol{w})^{-1}\nabla\mathcal{L}$~\cite{amari1998natural}.
If a preconditioning matrix is a Riemannian metric, it defines the geometry of the underlying structure and enables steepest descent learning.
Because we prove that $\mathbf{P}_{\text{GAP}}$ is a Riemannian metric for each parameter in Theorem~\ref{thm:sec.3.2}, $\mathbf{P}_{\text{GAP}}$ is theoretically guaranteed to enable steepest descent learning on its corresponding parameter space.
$\mathbf{P}_{\text{GAP}}$ consists of two factors, a unitary matrix of the inner-loop gradient $\mathbf{U}_{\tau,k}$ and a meta-parameter $\mathbf{M}$; 
$\mathbf{M}$ enables us to reflect the shared geometry information across the tasks. 
Task-specific and path-dependent geometry information can be reflected in the metric through $\mathbf{U}_{\tau,k}$.
Two factors allow our Riemannian metric to have higher function complexity than a constant metric. 
For example, we can consider various structures other than a unit-sphere that corresponds to the constant metric of $+1$.
Even though $\mathbf{P}_{\text{GAP}}$ is guaranteed to be a Riemannian metric, it is crucial that the meta-learned $\mathbf{P}_{\text{GAP}}$ corresponds to the true parameter space or at least $\mathbf{P}_{\text{GAP}}$ is close enough to be useful. We will discuss this issue in Section~\ref{discussion1}.

\subsection{Approximate GAP: a low-computational approximation of GAP}
\label{sec:3.3}
As presented in Section~\ref{sec:3.1}, GAP uses an SVD operation. The SVD operation can be burdensome for large-scale networks because it implies that the computational cost can be significantly increased. 
In recent studies, the use of a large-scale architecture has been emphasized as the key factor for improving the performance of meta-learning~\cite{hu2022pushing}. 
To make use of GAP for large-scale architectures without causing a computational problem, we provide an efficient approximation, named \textit{Approximate GAP}, under an assumption. 
\begin{assumption}
    \label{assumption}
    The elements of the gradient matrix follow an i.i.d. normal distribution with zero means.
\end{assumption}

Following \cite{wiedemann2020dithered, m2021efficient}, we adopt the assumption to have the gradient matrix become orthogonal as $n$ increases. Although the utilization of the assumption is a limiting factor, we empirically confirmed that the row vectors of the gradient matrix are indeed asymptotically orthogonal as $n$ increases~(see Fig.~\ref{fig:approx_size_diff}).

For the assumption, our approximation can be established as the following.
\begin{theorem}
    \label{theorem:approx}
    Let $\mathbf{G}\in\mathbb{R}^{m\times n}$ be a gradient matrix and $\mathbf{\tilde{G}}$ be the gradient transformed by meta parameter $\mathbf{M}$. Under the Assumption~\ref{assumption}, as $n$ becomes large, $\mathbf{\tilde{G}}$ asymptotically becomes equivalent to $\mathbf{M}\mathbf{G}$ as follows:
    \begin{equation}
        \mathbf{\tilde{G}} \cong \mathbf{M}\mathbf{G}
    \end{equation}
\end{theorem}

Note that we have chosen the larger dimension of the gradient matrix as $n$ when reshaping with Eq.~(\ref{eqn:reshaping}).
The proof is provided in Section~\ref{subsec:proof_of_thm2}.
This approximation has a clear trade-off between scalability and adaptiveness. Approximate GAP efficiently reduces the computational cost as shown in Table~\ref{tab:computational_comparison}, whereas the preconditioner $\mathbf{P}_{\text{GAP}}$ becomes not adaptive but constant.
However, Approximate GAP still guarantees the preconditioner to be a Riemannian metric.
As shown in Table~\ref{tab:miniImageNet} \& \ref{tab:tiered-ImageNet}, Approximate GAP incurs a slight performance drop but it still achieves a high performance owing to the Riemannian metric property.

\begin{figure}[t!]
  \centering
  \includegraphics[width=0.40\textwidth]{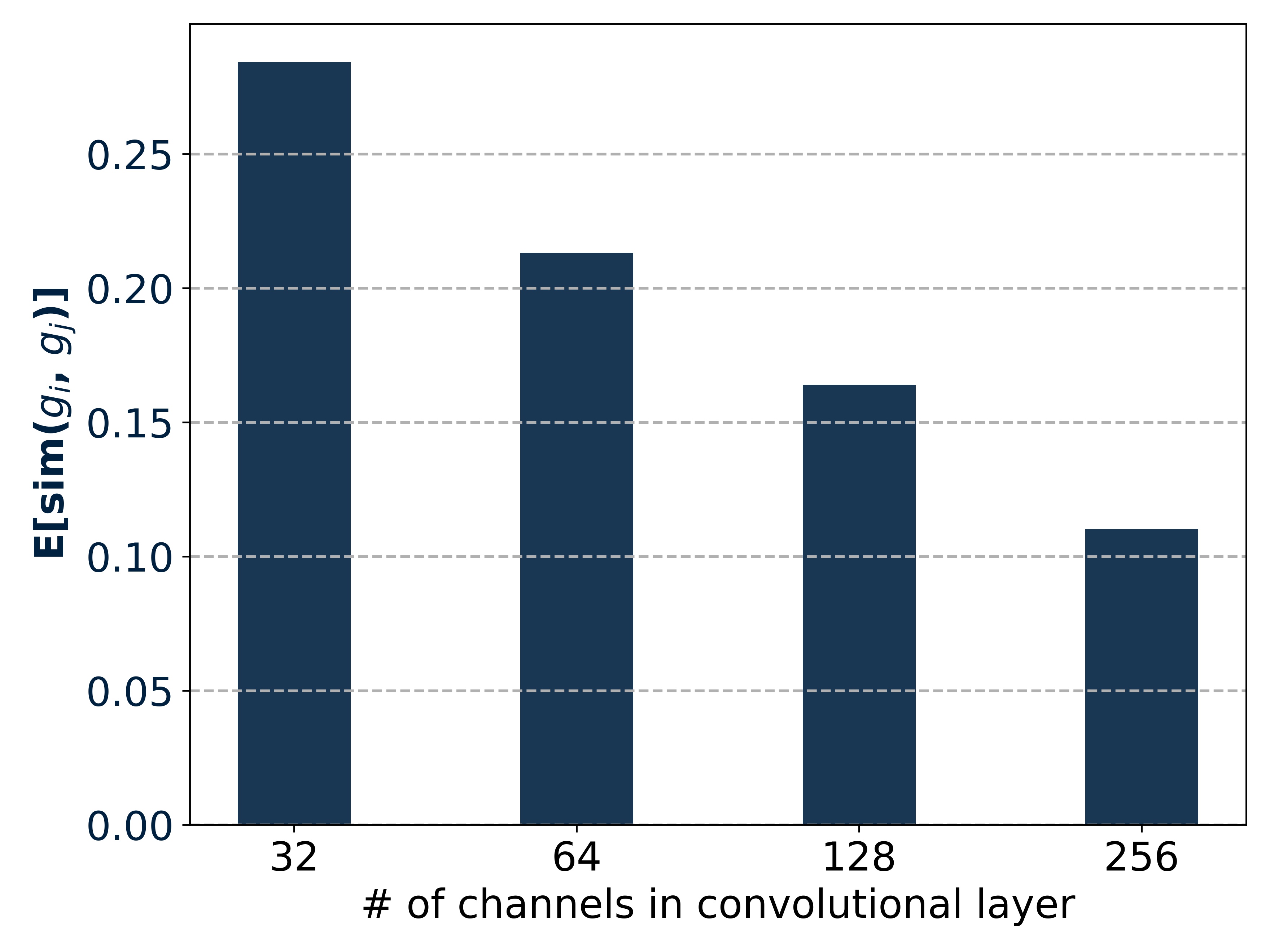}
  \caption{The average of cosine similarity between row vectors of gradient matrix as $n$ increases.}
  \label{fig:approx_size_diff}
\end{figure}

\begin{table}[t!]
  \caption{Comparison of training time, GPU memory, and test time for MAML, GAP, and Approximate GAP. We performed the experiment with 5-way 1-shot on mini-ImageNet and used 600 tasks in the test. We used the standard Conv-4 backbone.}
  \centering
  \resizebox{0.5\textwidth}{!}{\begin{tabular}{lccc}
        \hline
        Algorithm         & train time (secs)              & test time (secs)      & GPU-memory (MB) \\
        \hline
        MAML              & $35796.61$ & $52.66$  & $10185$\\ 
        GAP              & $50916.63$ & $120.44$ & $10279$\\
        Approximate GAP    & $36684.72$ & $53.62$  & $10217$\\
        \hline
  \end{tabular}
  }
  \label{tab:computational_comparison}
\end{table}

%% file: 3_Proofs.tex
\section{Proofs of theorems}
\label{proofs}
In this section, we provide proofs of the theorems stated in Section~\ref{sec:3.2} and \ref{sec:3.3}. We first provide a definition and three lemmas, and provide the proofs.
\subsection{A matrix similarity}
In linear algebra, the matrix similarity can be defined as below~\cite{strang2012linear}.
\begin{definition}
\label{def:similar_matrix}
    Two $n \times n$ matrices $A$ and $B$ are similar if there exists an invertible $n \times n$ matrix $P$ such that
    \begin{equation}
        B=P^{-1}AP
    \end{equation}
\end{definition}
Similar matrices represent the same linear transformation under different bases, with $P$ being the change of basis matrix. Therefore, similar matrices share all the properties of their common underlying transformation, such as rank, determinant, eigenvalues, and their algebraic multiplicities.

\subsection{Three lemmas}
In this section, we prove three lemmas that will be utilized in the main proofs.
\begin{lemma}
\label{lem:block}
    Let $A=\text{blkdiag}(A_1,\cdots,A_n)$ be a block diagonal matrix such that the main-diagonal blocks $A_i$ are $k \times k$ positive definite matrices.
    Then $A$ is a positive definite matrix.
\end{lemma}

\begin{proof}
    First, we show that $A$ is a positive definite matrix. For all non-zero $x=(x_1,\cdots,x_n)\in\mathbb{R}^{nk}$ where $x_i\in\mathbb{R}^k$, we can derive the following. 
    \begin{equation}
        \begin{split}
            x^T A x & = x^T \text{blkdiag}(A_1,\cdots,A_n) x \\
                    & = x_1^T A_1 x_1 + \cdots x_n^T A_n x_n \\
                    & > 0\;(\because A_i\text{ is a positive definite})
        \end{split}
    \end{equation}
    Next, we show that $A$ is a symmetric matrix. Since $A_i$ is a symmetric matrix (i.e., $A_i=A_i^T$), we find that the following is satisfied. 
    \begin{equation}
        \begin{split}
            A^T & = \text{blkdiag}(A_1,\cdots,A_n)^T\\
                & = \text{blkdiag}(A_1^T,\cdots,A_n^T)\\
                & = \text{blkdiag}(A_1,\cdots,A_n)\\
                & = A
        \end{split}
    \end{equation}
    Hence, $A$ is a symmetric matrix.
    Therefore, $A$ is a positive definite matrix.
\end{proof}

\begin{lemma}
\label{lemma1}
If a random vector $\bm{x}=(X_1,\cdots,X_n)$ $\in \mathbb{R}^n$ follows an uniform distribution on the $(n-1)$-dimensional unit sphere, the variance of the random variable $X_i$ satisfies the following.
\begin{equation}
    \mathbb{V}(X_{i}) = \frac{1}{n}
\end{equation}
\end{lemma}
\begin{proof}
Since $X_{1}, \cdots, X_{n}$ follow an identical distribution, $\mathbb{V}(X_{i}) = \mathbb{V}(X_{j})$ holds for all $i, j$.
Thus,
\begin{equation}
\label{eq:nvar}
    \begin{split}
        n\mathbb{V}(X_{i}) &= \sum_{i=1}^{n}\mathbb{V}(X_{i}).
    \end{split}
\end{equation}
Then, we derive the sum of variance as follows:
\begin{equation}
\label{eq:sumvar}
    \begin{split}
        \sum_{i=1}^{n}\mathbb{V}(X_{i})
    &= \sum_{i=1}^{n}\mathbb{E}(X_{i}^2) \text{ ($\because \mathbb{E}(X)=0$)} \\
    &= \mathbb{E}(\sum_{i=1}^{n}X_{i}^2) \\
    &= \mathbb{E}(||X||^2_2) \\
    &= 1.
    \end{split}
\end{equation}
By Eq.~(\ref{eq:nvar}) and~(\ref{eq:sumvar}), we have
\begin{equation}
    \begin{split}
        \mathbb{V}(X_{i}) &= \frac{1}{n}.
    \end{split}
\end{equation}
\end{proof} 

\begin{lemma}
\label{lemm2}
If two independent random vectors $\bm{x}=(X_1,\cdots,X_n)$, $\bm{y}=(Y_1,\cdots,Y_n)$ $\in \mathbb{R}^n$ follow a uniform distribution on the $(n-1)$-dimensional unit sphere, then 
\begin{equation}
    P(|\langle \bm{x}, \bm{y}\rangle| > \epsilon) \leq \frac{1}{n\epsilon^2}.
\end{equation}
\end{lemma}
\begin{proof}
Since we can rotate coordinate so that $\bm{y}=(1,0,\cdots,0) \in \mathbb{R}^{n}$, we have
\begin{equation}
\label{eq:simple}
    \langle \bm{x}, \bm{y} \rangle = X_1.
\end{equation}
Following Eq.~(\ref{eq:simple}), we show that its expectation is equal to:
\begin{equation}
    \begin{split}
        \mathbb{E}[\langle \bm{x}, \bm{y} \rangle] &= \mathbb{E}[X_1],\\
        &= 0
    \end{split}
\end{equation}
and its variance is equal to:
\begin{equation}
   \begin{split}
        \mathbb{V}[\langle \bm{x}, \bm{y} \rangle] &= \mathbb{V}[X_1],\\
    &= \frac{1}{n} \text{\ \ (by Lemma~\ref{lemma1})}.
   \end{split}
\end{equation}
By applying Chebyshev's inequality~\cite{bienayme1853considerations} on $\langle \bm{x}, \bm{y} \rangle$, we have
\begin{equation}
    \label{lemma:conclusion}
    P(|\langle \bm{x}, \bm{y}\rangle| \geq \frac{k}{\sqrt{n}}) \leq \frac{1}{k^2},
\end{equation} for any real number $k>0$. Let $\frac{k}{\sqrt{n}}$ be a $\epsilon$. Then we rewrite the in Eq.~(\ref{lemma:conclusion}) as follows:
\begin{equation}
    P(|\langle \bm{x}, \bm{y}\rangle| \geq \epsilon) \leq \frac{1}{n\epsilon^2}.
\end{equation}
This result indicates that the two vectors $\bm{x}$ and $\bm{y}$ become asymptotically orthogonal as $n$ increases.
\end{proof}

\subsection{Proof of Theorem~1}
\label{subsec:proof_of_thm1}
\begin{proof}
We can rewrite the $\mathbf{\tilde{G}}^l_{\tau,k}$ as follows:
\begin{equation}
\label{eq:thm1_eqn1}
    \begin{split}
        \mathbf{\tilde{G}}^{l}{}_{\tau, k}
    & = \mathbf{U}^{l}{}_{\tau, k}(\mathbf{M}^{l}\cdot\mathbf{\Sigma}^{l}{}_{\tau, k})\mathbf{V}^{l}{}_{\tau, k}{}^{T}\\
    & = (\mathbf{U}^{l}{}_{\tau, k}\mathbf{M}^{l}\mathbf{U}^{l}{}_{\tau, k}{}^{T})\mathbf{U}^{l}{}_{\tau, k}\mathbf{\Sigma}^{l}{}_{\tau, k}\mathbf{V}^{l}{}_{\tau, k}{}^{T}\\
    & = \mathbf{D}^{l}{}_{\tau, k}\mathbf{G}^{l}{}_{\tau, k},
    \end{split}
\end{equation}
where $\mathbf{D}^{l}{}_{\tau, k}=\mathbf{U}^{l}{}_{\tau, k}\mathbf{M}^{l}\mathbf{U}^{l}{}_{\tau, k}{}^{T}$.
To induce preconditioner in Eq.~(\ref{eq:thm1_eqn1}), we reformulate Eq.~(\ref{eq:thm1_eqn1}) as the general gradient descent form (i.e., matrix-vector product):
\begin{equation}
\label{eq:17}
    \begin{split}
        \mathbf{\tilde{G}}^{l}{}_{\tau, k}
    & = \text{blkdiag}(\underbrace{\mathbf{D}^{l}{}_{\tau, k}, \cdots, \mathbf{D}^{l}{}_{\tau, k}}_\text{$n$ times})\cdot(\mathbf{G}^{l}{}_{\tau, k})\\
    & = \mathbf{P}_{\text{GAP}}\cdot(\mathbf{G}^{l}{}_{\tau, k})
    \end{split}
\end{equation}
where $\mathbf{P}_{\text{GAP}}$ is a block diagonal matrix such that the main-diagonal blocks are $\mathbf{D}^{l}{}_{\tau, k}$'s. 
Now, we prove that block $\mathbf{D}^{l}{}_{\tau, k}$ is a positive definite matrix.
Since $\mathbf{D}^l_{\tau, k}$ is similar to $\mathbf{M}^l$ by Definition~\ref{def:similar_matrix}, they have the same eigenvalues. 
In addition, all eigenvalues of $\mathbf{D}^l_{\tau, k}$ are positive because all eigenvalues of $\mathbf{M}^l$ are positive.
Next, we show that $\mathbf{D}^l_{\tau, k}$ is a symmetric matrix as below. 
\begin{equation}
    \begin{split}
        (\mathbf{D}^{l}{}_{\tau, k}){}^{T} & = (\mathbf{U}^{l}{}_{\tau, k}\mathbf{M}^{l}\mathbf{U}^{l}{}_{\tau, k}{}^{T}){}^{T}\\
                                   & = \mathbf{U}^{l}{}_{\tau, k}\mathbf{M}^{l}\mathbf{U}^{l}{}_{\tau, k}{}^{T}\\
                                   & = \mathbf{D}^{l}{}_{\tau, k}
    \end{split}
\end{equation}
Therefore, $\mathbf{D}^l_{\tau, k}$ is a positive definite matrix. 
By Lemma~\ref{lem:block}, $\mathbf{P}_{\text{GAP}}$ is a positive definite matrix.

Because the unitary matrix $\mathbf{U}^l_{\tau, k}$ depends on the gradient matrix $\mathbf{\tilde{G}}^l_{\tau,k}$, it depends on the task-wise parameters $\theta_{\tau,k}$. 
Hence, $\mathbf{P}_{\text{GAP}}$ depends on the task-wise parameters $\theta_{\tau,k}$ because it depends on the unitary matrix $\mathbf{U}^l_{\tau, k}$.

Because $\mathbf{P}_{\text{GAP}}$ depends on the task-wise parameters $\theta_{\tau,k}$, it can be expressed as a function which is a smooth function mapping from the given $\theta_{\tau,k}$ to a positive definite matrix $\text{blkdiag}(\mathbf{D}^l_{\tau, k}, \cdots, \mathbf{D}^l_{\tau, k})$.
Hence, $\mathbf{P}_{\text{GAP}}$ is a Riemannian metric.

In summary, $\mathbf{P}_{\text{GAP}}$ is a Riemannian metric and depends on the task-specific parameters $\theta_{\tau,k}$.
\end{proof}

\subsection{Proof of Theorem~2}
\label{subsec:proof_of_thm2}
\begin{proof}
Let $\bm{g}_1, \bm{g}_2, \cdots, \bm{g}_m$ are the row vectors of $\mathbf{G}$. Then,
\begin{equation}
\begin{split}
\mathbf{G} 
=
\begin{bmatrix}
\lVert \bm{g}_1 \rVert    &        &             \\ 
                       & \ddots &             \\
                       &        & \lVert \bm{g}_m \rVert
\end{bmatrix}
\begin{bmatrix}
\frac{\bm{g}_1}{\lVert \bm{g}_1 \rVert} \\
\vdots \\
\frac{\bm{g}_m}{\lVert \bm{g}_m \rVert}
\end{bmatrix}.
\end{split}
\end{equation}
Under the Assumption~\ref{assumption}~(See Section~\ref{sec:3.3}), $\bm{g}_1, \bm{g}_2, \cdots, \bm{g}_m$ follow an i.i.d multivariate normal distribution.
Then, we have 
\begin{equation}
    \begin{split}
        \frac{\bm{g}_i}{\lVert \bm{g}_i \rVert} \ind \frac{\bm{g}_j}{\lVert \bm{g}_j \rVert}\ \ (\forall i \neq j),
    \end{split}
\end{equation}
and $\frac{\bm{g}_i}{\lVert \bm{g}_i \rVert}, \frac{\bm{g}_j}{\lVert \bm{g}_j \rVert}$ are located on the $(n-1)$-dimensional unit sphere~\cite{marsaglia1972choosing}.
Since independent vectors $\frac{\bm{g}_i}{\lVert \bm{g}_i \rVert}, \frac{\bm{g}_j}{\lVert \bm{g}_j \rVert}$ are located on the $(n-1)$-dimensional unit sphere, the vectors are asymptotically orthogonal as $n$ increases by Lemma~\ref{lemma1}.
Now, we rewrite $\mathbf{G}$ as follows.
\begin{equation}
\label{eqn:approx_final}
    \begin{split}
    \mathbf{G} 
    =
    \mathbf{I} 
        \begin{bmatrix}
            \lVert \bm{g}_1 \rVert    &        &             \\ 
                                   & \ddots &             \\
                                   &        &\lVert \bm{g}_m \rVert
        \end{bmatrix}
        \begin{bmatrix}
            \frac{\bm{g}_1}{\lVert \bm{g}_1 \rVert} \\
            \vdots \\
            \frac{\bm{g}_m}{\lVert \bm{g}_m \rVert}
        \end{bmatrix}
    \end{split}
\end{equation}
Because $\mathbf{I}$ is a unitary matrix and $(\frac{\bm{g}_1}{\lVert \bm{g}_1 \rVert}, \cdots, \frac{\bm{g}_m}{\lVert \bm{g}_m \rVert})^{T}$ approximately becomes semi-unitary matrices as $n$ increases, the singular values of $\mathbf{G}$ asymptotically become $\lVert \bm{g}_1 \rVert, \cdots, \lVert \bm{g}_m \rVert$.

Because of Eq.~(\ref{eqn:approx_final}), the following holds under the Assumption~\ref{assumption}~(See Section~\ref{sec:3.3}) as $n$ becomes sufficiently large.
\begin{equation}
\label{eq:33}
\begin{split}
\mathbf{\tilde{G}} \cong \mathbf{M}\mathbf{G}
\end{split}
\end{equation}
\end{proof}

%% file: 4_Experiments.tex
\section{Experiments}
In this section, we show the superiority of GAP by comparing it with the state-of-the-art PGD-MAML family and the MAML family. 

\subsection{Implementation Details}
For the reproducibility, we provide the details of implementation. Our implementations are based on Torchmeta~\cite{deleu2019torchmeta} library. 

\subsubsection{Hyper-parameters for the few-shot learning}
\label{sec:5.1.1}
For the few-shot learning experiments, we use the hyper-parameters in Table~\ref{hyperparameters1}.
\begin{table}[h!]
  \caption{Hyper-parameters used for training GAP on various few-shot learning experiments.}
  \centering
  \resizebox{0.99\columnwidth}{!}{\begin{tabular}{lccccccccc}
    \toprule
    Hyper-parameter                 &\multicolumn{3}{c}{Sinusoid} &\multicolumn{2}{c}{mini-ImageNet} &\multicolumn{2}{c}{tiered-ImageNet} & \multicolumn{2}{c}{Cross-domain}  \\
    \midrule
                                        & 5 shot & 10 shot & 20 shot & 1 shot & 5 shot      & 1 shot & 5 shot    & 1 shot & 5 shot\\
    Bathc size                          & 4     & 4       & 4        & 4     & 2            & 4     & 2          & 4     & 2      \\
    Total training iteration            & 70000 & 70000   & 70000    & 80000 & 80000        & 130000& 200000     & 80000 & 80000  \\
    inner learning rate $\alpha$        & $10^{-2}$ & $10^{-2}$ & $10^{-2}$ & $10^{-2}$  & $10^{-2}$ & $10^{-2}$ & $10^{-2}$ & $10^{-2}$ & $10^{-2}$ \\
    outer learning rate $\beta_1$       & $10^{-3}$ & $10^{-3}$ & $10^{-3}$ & $10^{-4}$ & $10^{-4}$ & $10^{-4}$ & $10^{-4}$ & $10^{-4}$ & $10^{-4}$ \\
    outer learning rate $\beta_2$       & $10^{-3}$ & $10^{-3}$ & $10^{-4}$   & $3\times10^{-3}$ & $10^{-4}$ & $3\times10^{-3}$ & $10^{-4}$ & $3\times10^{-3}$ & $10^{-4}$ \\
    The number of training inner steps  & 5     & 5       & 5        & 5     & 5            & 5     & 5          & 5     & 5      \\
    The number of testing inner steps   & 10    & 10      & 10       & 10    & 10           & 10    & 10         & 10    & 10     \\
    Data augmentation                   & \multicolumn{3}{c}{None}   & \multicolumn{2}{c}{random flip} & \multicolumn{2}{c}{random flip} & \multicolumn{2}{c}{random flip} \\
    \bottomrule
  \end{tabular}}
  \label{hyperparameters1}
\end{table}

\subsubsection{Hyper-parameters for the few-shot domain generalization and reinforcement learning}
For the few-shot domain generalization and reinforcement learning experiments, we use the hyper-parameters in Table~\ref{hyperparameters2}. In the few-shot domain generalization experiments, we use the data augmentation used in Meta-Dataset~\cite{triantafillou2019meta}.
\begin{table}[h!]
  \caption{Hyper-parameters used for training GAP on the few-shot domain generalization and reinforcement learning experiments.}
  \centering
  \resizebox{0.99\columnwidth}{!}{\begin{tabular}{lccccccc}
    \toprule
    Hyper-parameter                 &\multicolumn{4}{c}{Meta-Dataset} & Half-cheetah Dir & Half-cheetah Vel & 2D Navigation  \\
    \midrule
                                        &\multicolumn{2}{c}{ImageNet only} &\multicolumn{2}{c}{All dataset} \\
                                        & +fo-MAML  & +Proto-MAML   & +fo-MAML  & +Proto-MAML      \\
    Bathc size                          & 16        & 16            & 16        & 16         & 40        & 40        & 40 \\
    Total training iteration            & 50000     & 50000         & 50000     & 50000      & 500       & 500       & 500\\   
    inner learning rate $\alpha$        & $10^{-1}$ & $10^{-1}$     & $10^{-1}$ & $10^{-1}$  & $10^{-1}$ & $10^{-1}$ & $10^{-1}$ \\
    outer learning rate $\beta_1$       & $10^{-5}$ & $10^{-4}$     & $10^{-5}$ & $10^{-4}$  & $10^{-3}$ & $10^{-3}$ & $10^{-4}$\\
    outer learning rate $\beta_2$       & $10^{-5}$ & $10^{-4}$     & $10^{-5}$ & $10^{-5}$  & $10^{-3}$ & $10^{-3}$ & $10^{-4}$\\
    The number of training inner steps  & 10        & 10            & 10        & 10         & 1, 2, 3   & 1, 2, 3   & 1, 2, 3\\
    The number of testing inner steps   & 40        & 40            & 40        & 40         & 1, 2, 3   & 1, 2, 3   & 1, 2, 3\\
    \bottomrule
  \end{tabular}}
  \label{hyperparameters2}
\end{table}

\subsubsection{Backbone Architecture}
\label{sec:5.1.2}
\textbf{2-layer MLP network:}
For the few-shot regression experiment, we use a simple Multi-Layer Perceptron (MLP) with two hidden layers of size $40$, with ReLU  nonlinearities as in~\cite{finn2017model}. For the reinforcement learning experiment, we also use a simple MLP with two hidden layers of size $100$ with ReLU as in~\cite{finn2017model}.
\newline
\textbf{4-Conv network:}
For the few-shot classification and cross-domain few-shot classification experiments, we use the standard Conv-4 backbone used in~\cite{vinyals2016matching}, comprising 4 modules with $3 \times 3$ convolutions, with 128 filters followed by batch normalization~\cite{ioffe2015batch}, ReLU non-linearity, and $2 \times 2$ max-pooling.
\newline
\textbf{ResNet-18:}
For the few-shot domain generalization experiments, we employ ResNet-18 as the general feature extractor, following the methodology of previous few-shot domain generalization studies~\cite{triantafillou2019meta, baik2023learning}. 
In GAP+fo-MAML, the weights and biases of the linear layer are initialized to zero and are not meta-trained, consistent with~\cite{triantafillou2019meta}. This means that the linear layer is adapted from zero initialization during the inner-loop optimization. 
In GAP+Proto-MAML, the linear layer is initialized as described in~\cite{triantafillou2019meta}.

\subsubsection{Optimization}
We use ADAM optimizer~\cite{kingma2014adam}. 
For tiered-ImageNet experiment, the learning rate (LR) is scheduled by the cosine learning rate decay~\cite{loshchilov2016sgdr} for every 500 iterations. 
For Meta-Dataset, we decay the learning rate of each parameter by a factor of $0.8$ every $10,000$ iterations. In all the experiments except for tiered-ImageNet and Meta-Dataset, the learning rate is unscheduled.

\subsubsection{Pre-training}
For the few-shot domain generalization experiments, we initialize the feature extractor using the weights of the k-NN Baseline model trained on ImageNet, as described in~\cite{triantafillou2019meta}.

\subsubsection{Preconditioning}
In the few-shot regression and the reinforcement learning experiment, we apply preconditioner only to the hidden layer. In both few-shot classification and the few-shot domain generalization experiments, we only apply preconditioner to convolutional layers.

\subsection{Few-shot regression}
\subsubsection{Datasets and Experimental setup}
The goal of few-shot regression is to fit an unknown target function for the given $K$ sample points from the function.
For the evaluation of few-shot regression, we use the sinusoid regression benchmark~\cite{finn2017model}.
In this benchmark, sinusoid is used as the target function. Each task has a sinusoid $y(x)=A\sin(\omega x + b)$ as the target function, where the parameter values are within the following range: amplitude $A\in\left[0.1, 5.0\right]$, frequency $\omega\in\left[0.8, 1.2\right]$, and phase $b\in\left[0, \pi\right]$. For each task, input data point $x$ is sampled from $\left[-5.0, 5.0\right]$. 
In the experiment, we use a simple Multi-Layer Perceptron (MLP), following the setting in~\cite{finn2017model}. 
The details of the architecture are provided in Section~\ref{sec:5.1.2}.

\subsubsection{Results}
We evaluate GAP and compare it with MAML family and PGD-MAML family on a regression task. As shown in Table~\ref{tab:regression}, GAP consistently achieves the lowest mean squared error (MSE) scores, with the lowest confidence intervals in all three cases. The performance of GAP is improved by 89\% on 10-shot and 94\% on 20-shot compared to the performance of state-of-the-art algorithms.

\begin{table}[t!]
    \caption{Few-shot regression for the sinusoid regression benchmark with a \textit{2-layer MLP} backbone. We report MSE $\pm$ 95\% confidence intervals(ci) for 600 tasks following the setup in \cite{finn2017model}. $^{\dagger}$ denotes PGD-MAML family.}
  \centering
  \resizebox{0.5\textwidth}{!}{\begin{tabular}{lcccc}
    \toprule
    Algorithm                           &\quad\; 5-shot         &\quad\; 10-shot             &\quad\; 20-shot\\
    \midrule
    MAML~\cite{finn2017model}          &\quad\; $1.13\pm0.18$  &\quad\; $0.77\pm0.11$       &\quad\; $0.48\pm0.08$\\
    Meta-SGD$^{\dagger}$~\cite{li2017meta}         &\quad\; $0.90\pm0.16$  &\quad\; $0.53\pm0.09$       &\quad\; $0.31\pm0.05$\\
    MT-Net~\cite{lee2018gradient}      &\quad\; $0.76\pm0.09$  &\quad\; $0.49\pm0.05$       &\quad\; $0.33\pm0.04$\\
    ALFA~\cite{baik2020meta}           &\quad\; $0.92\pm0.19$  &\quad\; $0.62\pm0.16$       &\quad\; $0.34\pm0.07$\\
    L2F~\cite{baik2020learning}        &\quad\; $0.71\phantom{.} \pm $ N/A   &\quad\; $0.37\phantom{.} \pm $ N/A  &\quad\; $0.16\phantom{.} \pm $ N/A\\
    PAMELA$^{\dagger}$~\cite{rajasegaran2020meta}  &\quad\; $0.54\pm0.06$  &\quad\; $0.41\pm0.04$       &\quad\; $0.17\pm0.03$\\
    MeTAL~\cite{baik2021meta}          &\quad\; $0.74\pm0.18$  &\quad\; $0.44\pm0.11$       &\quad\; $0.21\pm0.06$\\
    \midrule
    GAP$^{\dagger}$                                &\quad\; $\mathbf{0.16\pm0.04}$  &\quad\; $\mathbf{0.04\pm0.01}$       &\quad\; $\mathbf{0.01\pm0.01}$\\
    \bottomrule
  \end{tabular}}
  \label{tab:regression}
\end{table}

\subsection{Few-shot classification}
\subsubsection{Datasets and Experimental setup}
For the few-shot classification, we evaluate two benchmarks:
(1) mini-ImageNet~\cite{vinyals2016matching}; this dataset has 100 classes and it is a subset of ImageNet~\cite{russakovsky2015imagenet}, and we use the same split as in~\cite{ravi2016optimization}, with 64, 16 and 20 classes for train, validation and test, respectively. (2) tiered-ImageNet~\cite{ren2018meta}; this is also a subset of ImageNet with 608 classes grouped into 34 high-level categories, and divided into 20, 6 and 8 for train, validation, and test, respectively.
For all the experiments, our model follows the standard \textit{Conv-4} backbone used in~\cite{vinyals2016matching}.
The details of the architecture are provided in Section~\ref{sec:5.1.2}.
Following the experimental protocol in~\cite{finn2017model}, we use 15 samples per class in the query-set to compute the meta gradients. In meta training and meta testing, the inner-loop optimization is updated in five and ten steps, respectively.

\subsubsection{Results}
Table~\ref{tab:miniImageNet}~\&~\ref{tab:tiered-ImageNet} present the performance of GAP, the state-of-the-art PGD-MAML family, and the state-of-the-art MAML-family on mini-ImageNet and tiered-ImageNet under two typical settings: 5-way 1-shot and 5-way 5-shot. 
The GAP outperforms all of the previous PGD-MAML family and MAML family. 
Compared to the state-of-the-art MAML family, GAP improves the performance with a quite significant margin for both mini-ImageNet and tiered-ImageNet datasets.
Compared to the state-of-the-art PGD-MAML family, GAP shows that the 1- and 5-shot accuracy can be increased by 1.4~\% and 1.5~\% on mini-ImageNet dataset, and by 0.7~\% and 0.68~\% on tiered-ImageNet dataset, respectively. 
We also evaluated Approximate GAP that is introduced in Section~\ref{sec:3.3}. 
The results show that the approximated version can perform comparably to the original GAP. 
Although Approximate GAP shows slightly lower accuracies than the original,  its performance is superior to most of the existing algorithms because of its Riemannian metric property.

\begin{table}[t!]
    \caption{5-way few-shot classification accuracy (\%) on mini-ImageNet with a \textit{Conv-4} backbone. We report mean $\pm$ 95\% confidence intervals(ci) for 600 tasks according to \cite{finn2017model}. $^{\dagger}$ denotes PGD-MAML family.}
  \centering
  \resizebox{0.5\textwidth}{!}{\begin{tabular}{lccc}
    \toprule
    Algorithm                                          & \quad \; 5-way 1-shot                         & \quad \; 5-way 5-shot\\
    \midrule
    MAML~\cite{finn2017model}                           & \quad \; $47.89 \pm 1.20$                     & \quad \; $64.59 \pm 0.88$\\
    Meta-SGD$^{\dagger}$~\cite{li2017meta}                         & \quad \; $50.47 \pm 1.87$                     & \quad \; $64.00 \pm 0.90$\\
    BMAML~\cite{yoon2018bayesian}              & \quad \; $53.80 \pm 1.46$                     & \quad \; $64.23 \pm 0.69$\\
    ANIL~\cite{raghu2019rapid}                         & \quad \; $46.70 \pm 0.40$                     & \quad \; $61.50 \pm 0.50$\\
    LLAMA~\cite{grant2018recasting}.                   & \quad \; $49.40 \pm 1.83$                     & \quad \quad N/A\\
    PLATIPUS~\cite{finn2018probabilistic}              & \quad \; $50.13 \pm 1.86$                     & \quad \quad -\\
    T-net~\cite{lee2018gradient}                       & \quad \; $50.86 \pm 1.82$                     & \quad \quad N/A\\
    MT-net~\cite{lee2018gradient}                      & \quad \; $51.70 \pm 1.84$                    & \quad \quad N/A\\
    MAML++~\cite{antoniou2018train}                    & \quad \; $52.15 \pm 0.26$                     & \quad \; $68.32 \pm 0.44$\\ 
    iMAML-HF~\cite{rajeswaran2019meta}                 & \quad \; $49.30 \pm 1.88$                     & \quad \quad N/A\\
    WarpGrad~\cite{flennerhag2019meta}                 & \quad \; $52.30 \pm 0.90$                     & \quad \; $68.40 \pm 0.60$\\
    MC1$^{\dagger}$~\cite{park2019meta}                            & \quad \; $53.74 \pm 0.84$                     & \quad \; $68.01 \pm 0.73$\\
    MC2$^{\dagger}$~\cite{park2019meta}                            & \quad \; $54.08 \pm 0.88$                     & \quad \; $67.99 \pm 0.73$\\
    MH-C$^{\dagger}$~\cite{zhao2020meta}                           & \quad \; $48.64 \pm 0.33$                     & \quad \; $64.52 \pm 0.51$\\
    MH$^{\dagger}$~\cite{zhao2020meta}                             & \quad \; $49.41 \pm 0.96$                     & \quad \; $67.16 \pm 0.42$\\
    BOIL~\cite{oh2020boil}                             & \quad \; $49.61 \pm 0.16$                     & \quad \; $66.46 \pm 0.37$\\

    ARML~\cite{yao2020automated}                       & \quad \; $50.42 \pm 1.79$                     & \quad \; $64.12 \pm 0.90$\\
    ALFA~\cite{baik2020meta}                           & \quad \; $50.58 \pm 0.51$                     & \quad \; $69.12 \pm 0.47$\\

    L2F~\cite{baik2020learning}                        & \quad \; $52.10 \pm 0.50$                     & \quad \; $69.38 \pm 0.46$\\
    ModGrad$^{\dagger}$~\cite{simon2020modulating}                 & \quad \; $53.20 \pm 0.86$                     & \quad \; $69.17 \pm 0.69$\\
    PAMELA$^{\dagger}$~\cite{rajasegaran2020meta}                  & \quad \; $53.50 \pm 0.89$                     & \quad \; $70.51 \pm 0.67$\\
    SignMAML~\cite{fan2021sign}                        & \quad \; $42.90 \pm 1.50$                     & \quad \; $60.70 \pm 0.70$\\
    CTML~\cite{peng2021clustered}                      & \quad \; $50.47 \pm 1.83$                     & \quad \; $64.15 \pm 0.90$\\
    MeTAL~\cite{baik2021meta}                        & \quad \; $52.63 \pm 0.37$                     & \quad \; $70.52 \pm 0.29$\\
    ECML~\cite{hiller2022enforcing}                    & \quad \; $48.94 \pm 0.80$                     & \quad \; $65.26 \pm 0.67$\\
    Sharp-MAML\_{up}~\cite{abbas2022sharp}              & \quad \; $49.56\phantom{.} \pm $ N/A          & \quad \; $63.06\phantom{.} \pm $ N/A\\
    Sharp-MAML\_{low}~\cite{abbas2022sharp}             & \quad \; $49.72\phantom{.} \pm $ N/A          & \quad \; $63.18\phantom{.} \pm $ N/A\\
    Sharp-MAML\_{both}~\cite{abbas2022sharp}            & \quad \; $50.28\phantom{.} \pm $ N/A          & \quad \; $65.04\phantom{.} \pm $ N/A\\
    FBM~\cite{yang2022calibrating}                     & \quad \; $50.62 \pm 1.79$                     & \quad \; $64.78 \pm 0.35$\\
    CxGrad~\cite{lee2022contextual}                    & \quad \; $51.80 \pm 0.46$                     & \quad \; $69.82 \pm 0.42$\\
    HyperMAML~\cite{przewikezlikowski2022hypermaml}     & \quad \; $51.84 \pm 0.57$                     & \quad \; $66.29 \pm 0.43$\\

    EEML~\cite{li2022eeml}                             & \quad \; $52.42 \pm 1.75$                     & \quad \; $68.40 \pm 0.95$\\

    MH-O$^{\dagger}$~\cite{zhao2020meta}                           & \quad \; $52.50 \pm 0.61$                     & \quad \; $67.76 \pm 0.34$\\
    Sparse-MAML$^{\dagger}$~\cite{von2021learning}                 & \quad \; $50.35 \pm 0.39$                     & \quad \; $67.03 \pm 0.74$\\
    Sparse-ReLU-MAML$^{\dagger}$~\cite{von2021learning}            & \quad \; $50.39 \pm 0.89$                     & \quad \; $64.84 \pm 0.46$\\
    Sparse-MAML+$^{\dagger}$~\cite{von2021learning}           & \quad \; $51.04 \pm 0.59$                     & \quad \; $67.03 \pm 0.74$\\
    \midrule
    Approximate GAP$^{\dagger}$                                    & \quad \; $53.52 \pm 0.88$                     & \quad \; $70.75 \pm 0.67$\\
    GAP$^{\dagger}$                                              & \quad \; $\mathbf{54.86} \pm \mathbf{0.85}$   & \quad \; $\mathbf{71.55}\pm\mathbf{0.61}$\\
    \bottomrule
  \end{tabular}}
  \label{tab:miniImageNet}
\end{table}

\begin{table}[t!]
    \caption{5-way few-shot classification accuracy (\%) on tiered-ImageNet dataset with a \textit{Conv-4} backbone. We report mean $\pm$ 95\% confidence intervals(ci) for 600 tasks according to \cite{finn2017model}. $^{\dagger}$ denotes PGD-MAML family.}
  \centering
  \resizebox{0.5\textwidth}{!}{\begin{tabular}{lccc}
    \toprule
    Algorithm                                       & \quad \; 5-way 1-shot                           & \quad \; 5-way 5-shot\\
    \midrule
    Meta-SGD$^{\dagger}$~\cite{li2017meta}                  & \quad \; $50.92 \pm 0.93$                     & \quad \; $69.28 \pm 0.80$\\
    MAML~\cite{finn2017model}                   & \quad \; $51.70 \pm 1.80$                     & \quad \; $70.30 \pm 1.80$\\
    MT-net~\cite{lee2018gradient}               & \quad \; $51.95 \pm 1.83$                     & \quad \quad N/A\\
    WarpGrad~\cite{flennerhag2019meta}           & \quad \; $57.20 \pm 0.90$                     & \quad \; $74.10 \pm 0.70$\\
    BOIL~\cite{oh2020boil}.                    & \quad \; $48.58 \pm 0.27$                     & \quad \; $69.37 \pm 0.12$\\
    ALFA~\cite{baik2020meta}                    & \quad \; $53.16 \pm 0.49$                     & \quad \; $70.54 \pm 0.46$\\
    L2F~\cite{baik2020learning}                 & \quad \; $54.40 \pm 0.50$                     & \quad \; $73.34 \pm 0.44$\\
    ARML~\cite{yao2020automated}                & \quad \; $52.91 \pm 1.83$                     & \quad \quad N/A\\
    PAMELA$^{\dagger}$~\cite{rajasegaran2020meta}           & \quad \; $54.81 \pm 0.88$                     & \quad \; $74.39 \pm 0.71$\\
    Sparse-ReLU-MAML$^{\dagger}$~\cite{von2021learning}      & \quad \; $53.18 \pm 0.52$                     & \quad \; $69.06 \pm 0.28$\\
    Sparse-MAML$^{\dagger}$~\cite{von2021learning}          & \quad \; $53.47 \pm 0.53$                     & \quad \; $68.83 \pm 0.65$\\
    Sparse-MAML+$^{\dagger}$~\cite{von2021learning}     & \quad \; $53.91 \pm 0.67$                     & \quad \; $69.92 \pm 0.21$\\
    MeTAL~\cite{baik2021meta}                   & \quad \; $54.34 \pm 0.31$                     & \quad \; $70.40 \pm 0.21$\\
    
    CxGrad~\cite{lee2022contextual}             & \quad \; $55.55 \pm 0.46$                     & \quad \; $73.55 \pm 0.41$\\
    ECML~\cite{hiller2022enforcing}             & \quad \; $47.34 \pm 0.88$                     & \quad \; $64.77 \pm 0.75$\\
    \midrule
   Approximate GAP$^{\dagger}$                                & \quad \; $56.86 \pm 0.91$                     & \quad \; $74.41 \pm 0.72$\\
   GAP$^{\dagger}$                                      & \quad \; $\mathbf{57.60} \pm \mathbf{0.93}$   & \quad \; $\mathbf{74.90} \pm \mathbf{0.68}$\\
    \bottomrule
  \end{tabular}}
  \label{tab:tiered-ImageNet}
\end{table}

\subsection{Cross-domain few-shot classification}
The cross-domain few-hot classification introduced by~\cite{chen2019closer} addresses a more challenging and practical few-shot classification scenario in which meta-train and meta-test tasks are sampled from different task distributions. These scenarios are designed to evaluate meta-level overfitting of meta-learning algorithms by creating a large domain gap between meta-trains and meta-tests. In particular, an algorithm can be said to be meta-overfitting if it relies too much on the prior knowledge of previously seen meta-train tasks instead of focusing on a few given examples to learn a new task. This meta-level overfitting makes the learning system more likely to fail to adapt to new tasks sampled from substantially different task distributions.

\subsubsection{Datasets and Experimental setup}
To evaluate the level of meta-overfitting for GAP, we evaluate a cross-domain few-shot classification experiment.
The mini-ImageNet is used for the meta-train task, and the tiered-ImageNet~\cite{ren2018meta}, CUB-200-2011~\cite{wah2011caltech}, Cars~\cite{bertinetto2018meta} datasets are used for the meta-test task. The CUB dataset has 200 fine-grained classes and consists of a total of 11,788 images; it is further divided into 100 meta-train classes, 50 meta-validation classes, and 50 meta-test classes. 
The Cars~\cite{krause20133d} dataset consists of 16,185 images of 196 classes of cars; it is split into 8,144 training images and 8,041 testing images, where each class has been split roughly in 50-50. The classes are typically at the level of Make, Model, Year, e.g., 2012 Tesla Model S or 2012 BMW M3 coupe.
As with the few-shot classification experiment, we use the standard \textit{Conv-4} backbone and follow the same experimental protocol.

\subsubsection{Results.}
Table~\ref{tab:cross_domain} presents the cross-domain few-shot performance for GAP, MAML family, and PGD-MAML family. GAP significantly outperforms the state-of-the-art algorithms on 5-way 1-shot and 5-way 5-shot cross-domain classification tasks. In particular, for the tiered-ImageNet dataset, the performance was improved by 8.6\% and 4.1\% on 1-shot and 5-shot classification tasks, respectively. Because GAP can simultaneously consider a task's individuality and optimization trajectory in the inner-loop optimization, it can overcome meta-overfitting better than the existing methods. 
However, Approximate GAP shows more performance degradation in cross-domain few-shot classification than in few-shot classification. 
In particular, when the domain difference with the meta-train is more significant (i.e., the tiered-ImageNet dataset) than when the domain difference with the meta-train is marginal (i.e., CARS and CUB datasets), it shows a more considerable performance drop. 
We can see that full adaptation plays an important role in cross-domain few-shot classification. 

\begin{table*}[t!]
    \caption{5-way few-shot cross domain classification accuracy (\%) with a \textit{Conv-4} backbone, meta training on mini-ImageNet dataset, and meta-testing on tiered-ImageNet, CUB, or Cars datasets. We report mean $\pm$ 95\% confidence intervals(ci) for 600 tasks according to \cite{finn2017model}. $^{\dagger}$ denotes PGD-MAML family.
  }
  \centering
  \resizebox{1.\textwidth}{!}{
  \begin{tabular}{lcccccc}
    \toprule                                     
                       & \multicolumn{2}{c}{tiered-ImageNet} & \multicolumn{2}{c}{CUB} & \multicolumn{2}{c}{Cars} \\
                       \cmidrule(l){2-3} \cmidrule(l){4-5} \cmidrule(l){6-7}
    Algorithm             & 1-shot  & 5-shot & 1-shot  & 5-shot & 1-shot  & 5-shot \\
    \midrule
    MAML~\cite{finn2017model}               & $51.61\pm0.20$ & $65.76\pm0.27$   & $40.51\pm0.08$ & $53.09\pm0.16$ & $33.57\pm0.14$ & $44.56\pm0.21$ \\
    ANIL~\cite{raghu2019rapid}               & $52.82\pm0.29$ & $66.52\pm0.28$   & $41.12\pm0.15$ & $55.82\pm0.21$ & $34.77\pm0.31$ & $46.55\pm0.29$ \\
    BOIL~\cite{oh2020boil}               & $53.23\pm0.41$ & $69.37\pm0.23$   & $44.20\pm0.15$ & $60.92\pm0.11$ & $36.12\pm0.29$ & $50.64\pm0.22$\\
    BMAML~\cite{yoon2018bayesian}              & N/A              & N/A            & $33.52\pm0.36$ & $51.35\pm0.16$ & N/A            & N/A \\
    ALFA~\cite{baik2020meta}               & N/A              & N/A            & N/A             & $58.35\pm0.25$ & N/A           & N/A \\
    L2F~\cite{baik2020learning}                & N/A              & N/A            & N/A             & $60.89\pm0.22$ & N/A           & N/A \\
    MeTAL~\cite{baik2021meta}              & N/A              & N/A            & N/A             & $58.20\pm0.24$ & N/A           & N/A \\
    HyperMAML~\cite{przewikezlikowski2022hypermaml}          & N/A              & N/A            & $36.52\pm0.61$ & $49.43\pm0.14$ & N/A            & N/A \\
    CxGrad~\cite{lee2022contextual}             & N/A              & N/A            & N/A             & $63.92\pm0.44$ & N/A           & N/A \\
    Sparse-MAML$^{\dagger}$~\cite{von2021learning}         & $53.47\pm0.53$ & $68.83\pm0.65$   & $41.37\pm0.73$ & $60.58\pm1.10$ & $35.90\pm0.50$ & $52.63\pm0.56$ \\
    Sparse-ReLU-MAML$^{\dagger}$~\cite{von2021learning}    & $53.77\pm0.94$ & $68.12\pm0.69$   & $42.89\pm0.45$ & $57.53\pm0.94$ & $36.04\pm0.55$ & $49.95\pm0.42$ \\
    Sparse-MAML+$^{\dagger}$~\cite{von2021learning}        & $53.91\pm0.67$ & $69.92\pm0.21$   & $43.43\pm1.04$ & $62.02\pm0.78$ & $37.14\pm0.77$ & $53.18\pm0.44$ \\
    \midrule
    Approximate GAP$^{\dagger}$   & $57.47 \pm 0.99$ & $71.66 \pm 0.76$ & $43.77 \pm 0.79$ & $62.92 \pm 0.73$ & $37.00 \pm 0.75$ & $53.28 \pm 0.76$\\
    GAP$^{\dagger}$             & $\mathbf{58.56\pm0.93}$ & $\mathbf{72.82\pm0.77}$   & $\mathbf{44.74\pm0.75}$ & $\mathbf{64.88\pm0.72}$ & $\mathbf{38.44\pm0.77}$ & $\mathbf{55.04\pm0.77}$ \\
    \bottomrule
  \end{tabular}
  }
  \label{tab:cross_domain}
\end{table*}

\subsection{Few-shot domain generalization}
\subsubsection{Datasets and Experimental setup}
We use the Meta-Dataset~\cite{triantafillou2019meta} which is the standard benchmark for the few-shot domain generalization. It is a large-scale benchmark that has been widely used in recent years for few-shot domain generalization through multiple domains. It contains a total of ten diverse datasets: ImageNet~\cite{russakovsky2015imagenet}, Omniglot~\cite{lake2015human}, FGVC-Aircraft~(Aircraft)~\cite{maji2013fine}, CUB-200-2011~(Birds)~\cite{wah2011caltech}, Describable Textures~(DTD)~\cite{cimpoi2014describing}, QuickDraw~\cite{jongejan2016quick}, FGVCx Fungi~(Fungi)~\cite{schroeder2018fgvcx}, VGG Flower~(Flower)~\cite{nilsback2008automated}, Traffic Signs~\cite{houben2013detection} and MSCOCO~\cite{lin2014microsoft}.
Following the previous works~\cite{requeima2019fast, bateni2020improved, li2021universal, triantafillou2021learning, li2022cross}, we also add three additional datasets including MNIST~\cite{lecun1998gradient}, CIFAR10~\cite{krizhevsky2009learning} and CIFAR100~\cite{krizhevsky2009learning}.
We follow the standard training procedure in~\cite{triantafillou2019meta} and consider both the `Training on all datasets'~(multi-domain learning) and `Training on ImageNet only'~(single-domain learning) settings. 
In ‘Training on all datasets’ setting, we follow the standard procedure and use the first eight datasets for meta-training, in which each dataset is further divided into train, validation and test set with disjoint classes.
While the evaluation within these datasets is used to measure the generalization ability in the seen domains, the remaining five datasets are reserved as unseen domains in meta-test for measuring the cross-domain generalization ability. In `Training on ImageNet only' setting, we follow the standard procedure and only use train split of ImageNet for meta-training. The evaluation of models is performed using the test split of ImageNet and the other 12 datasets. For all the experiments, we adopt \textit{ResNet-18}~\cite{he2016deep} as the general feature extractor following the previous few-shot domain generalization works~\cite{requeima2019fast, bateni2020improved, li2021universal, triantafillou2021learning, li2022cross}. We apply GAP to fo-MAML~\cite{finn2017model} and Proto-MAML~\cite{triantafillou2019meta}. Following the experimental protocol in~\cite{triantafillou2019meta}, we use 600 randomly sampled tasks for each dataset with varying number of ways and shots. In meta training and meta-testing, the inner-loop optimization is updated in ten and forty steps, respectively.

\subsubsection{Results}
For Approximate GAP, GAP, and MAML family, Table~\ref{tab:img_metadataset}~\&~\ref{tab:all_metadataset} present the performance of models trained on ImageNet only and trained on all dataset, respectively. 
The results demonstrate that GAP can consistently outperform fo-MAML~(first-order MAML) and Proto-MAML, which is a MAML variant proposed by~\cite{triantafillou2019meta} that substantially improves the MAML initialization at fc-layer using class prototypes. 
For models trained on ImageNet, GAP demonstrates that the performance of fo-MAML and Proto-MAML can be improved by $10.82\%$ and $7.13\%$ on the seen domains, and by $11.48\%$ and $7.85\%$ on the unseen domains without the three datasets~(MNIST, CIFAR-10, and CIFAR-100).
For models trained on all datasets, the performance of fo-MAML and Proto-MAML improved by $7.58\%$ and $2.87\%$ on the seen domains, and by $19.65\%$ and $10.31\%$ on the unseen domains when excluding the three datasets. 
Additionally, GAP+fo-MAML and GAP+Proto-MAML significantly outperform ALFA+fo-MAML and ALFA+Proto-MAML, which are known as part of the state-of-the-art MAML family, on both seen and unseen domains. 
An interesting aspect of these results is that Approximate GAP and GAP exhibit similar performance despite the domain differences, unlike in cross-domain few-shot classification. 
This similarity in performance can be attributed to Approximate GAP approximating GAP as the architecture size increases.

\begin{table*}[t!]
    \caption{
    Few-shot domain generalization for the Meta-Dataset benchmark with a ResNet-18 backbone trained on \textit{ImageNet only}. 
    The ImageNet is seen during meta-training and the other twelve datasets are unseen and used for test only. 
    We report the three average accuracy as follows: seen domains, unseen domains, all domains. 
    To ensure a fair comparison, we report the average accuracy for unseen domains and all domains under two settings: one including MNIST, CIFAR-10, and CIFAR-100 datasets and the other without them.}
  \centering
  \resizebox{1.\textwidth}{!}{\begin{tabular}{lccccccccccccc}
    \toprule
    Datasets &\quad\; fo-MAML &\quad\; Proto-MAML &\quad\; \begin{tabular}{@{}c@{}}ALFA\\+fo-MAML\end{tabular} &\quad\; \begin{tabular}{@{}c@{}}ALFA\\+Proto-MAML\end{tabular} &\quad\; \begin{tabular}{@{}c@{}}Approximate GAP\\+fo-MAML\end{tabular} &\quad\; \begin{tabular}{@{}c@{}}Approximate GAP\\+Proto-MAML\end{tabular} &\quad\; \begin{tabular}{@{}c@{}}GAP\\+fo-MAML\end{tabular} &\quad\; \begin{tabular}{@{}c@{}}GAP\\+Proto-MAML\end{tabular}\\
    \midrule
    ImageNet &\quad\; $45.51$ &\quad\; $49.53$ &\quad\; $51.09$ &\quad\; $52.80$ &\quad\; $55.34$ &\quad\; $56.65$ &\quad\; $56.33$ &\quad\; $\mathbf{56.66}$\\
    \midrule
    Omniglot &\quad\; $55.55$ &\quad\; $63.37$ &\quad\; $67.89$ &\quad\; $61.87$ &\quad\; $75.17$ &\quad\; $76.30$ &\quad\; $77.04$ &\quad\; $\mathbf{77.57}$\\
    Aircraft &\quad\; $56.24$ &\quad\; $55.95$ &\quad\; $66.34$ &\quad\; $63.43$ &\quad\; $67.40$ &\quad\; $67.61$ &\quad\; $68.03$ &\quad\; $\mathbf{68.50}$\\
    Birds &\quad\; $63.61$ &\quad\; $68.66$ &\quad\; $67.67$ &\quad\; $69.75$ &\quad\; $72.35$ &\quad\; $73.29$ &\quad\; $72.01$ &\quad\; $\mathbf{73.51}$\\
    Textures &\quad\; $68.04$ &\quad\; $66.49$ &\quad\; $65.34$ &\quad\; $70.78$ &\quad\; $71.12$ &\quad\; $71.01$ &\quad\; $\mathbf{71.82}$ &\quad\; $71.42$\\
    Quick Draw &\quad\; $43.96$ &\quad\; $51.52$ &\quad\; $60.53$ &\quad\; $59.17$ &\quad\; $63.49$ &\quad\; $63.71$ &\quad\; $64.36$ &\quad\; $\mathbf{65.36}$\\
    Fungi &\quad\; $32.10$ &\quad\; $39.96$ &\quad\; $37.41$ &\quad\; $41.49$ &\quad\; $37.19$ &\quad\; $38.30$ &\quad\; $37.46$ &\quad\; $\mathbf{38.57}$\\
    VGG Flower &\quad\; $81.74$ &\quad\; $87.15$ &\quad\; $84.28$ &\quad\; $85.96$ &\quad\; $84.44$ &\quad\; $86.73$ &\quad\; $85.31$ &\quad\; $\mathbf{86.79}$\\
    Traffic Sign &\quad\; $50.93$ &\quad\; $48.83$ &\quad\; $60.86$ &\quad\; $60.78$ &\quad\; $66.82$ &\quad\; $66.44$ &\quad\; $67.89$ &\quad\; $\mathbf{66.90}$\\
    MSCOCO &\quad\; $35.30$ &\quad\; $43.74$ &\quad\; $40.05$ &\quad\; $48.11$ &\quad\; $45.99$ &\quad\; $46.54$ &\quad\; $\mathbf{46.87}$ &\quad\; $46.76$\\
    MNIST$^{\star}$ &\quad\; N/A &\quad\; N/A &\quad\; N/A &\quad\; N/A &\quad\; $93.46$ &\quad\; $92.27$ &\quad\; $93.53$ &\quad\; $\mathbf{93.97}$\\
    CIFAR-10$^{\star}$ &\quad\; N/A &\quad\; N/A &\quad\; N/A &\quad\; N/A &\quad\; $75.08$ &\quad\; $74.05$ &\quad\; $\mathbf{75.96}$ &\quad\; $74.54$\\
    CIFAR-100$^{\star}$ &\quad\; N/A &\quad\; N/A &\quad\; N/A &\quad\; N/A &\quad\; $65.65$ &\quad\; $62.66$ &\quad\; $\mathbf{65.94}$ &\quad\; $63.23$\\
    \midrule
    Average seen &\quad\; $45.51$ &\quad\; $49.53$ &\quad\; $51.09$ &\quad\; $52.80$ &\quad\; $55.34$ &\quad\; $56.65$ &\quad\; $56.33$ &\quad\; $\mathbf{56.66}$\\
    \midrule
    Average unseen w/o $\star$ &\quad\; $54.16$ &\quad\; $58.41$ &\quad\; $61.15$ &\quad\; $62.37$ &\quad\; $64.89$ &\quad\; $65.65$ &\quad\; $65.64$ &\quad\; $\mathbf{66.26}$\\
    Average all w/o $\star$ &\quad\; $53.30$ &\quad\; $57.52$ &\quad\; $60.15$ &\quad\; $61.41$ &\quad\; $63.93$ &\quad\; $64.66$ &\quad\; $64.71$ &\quad\; $\mathbf{65.30}$\\
    \midrule
    Average unseen w/ $\star$ &\quad\; N/A &\quad\; N/A &\quad\; N/A &\quad\; N/A &\quad\; $69.56$ &\quad\; $68.39$ &\quad\; $\mathbf{70.04}$ &\quad\; $69.28$\\
    Average all w/ $\star$ &\quad\; N/A &\quad\; N/A &\quad\; N/A &\quad\; N/A &\quad\; $67.25$ &\quad\; $67.35$ &\quad\; $67.89$ &\quad\; $\mathbf{68.06}$\\
    \bottomrule
  \end{tabular}
  }
  \label{tab:img_metadataset}
\end{table*}

\begin{table*}[t!]
    \caption{
    Few-shot domain generalization for the Meta-Dataset benchmark with a ResNet-18 backbone trained on \textit{eight datasets}. 
    The first eight datasets are seen during meta-training and the other five datasets are unseen and used for test only. We report the three average accuracy as follows: seen domains, unseen domains, all domains. 
    To ensure a fair comparison, we report the average accuracy for unseen domains and all domains under two settings: one including MNIST, CIFAR-10, and CIFAR-100 datasets and the other without them.}
  \centering
  \resizebox{1.\textwidth}{!}{\begin{tabular}{lccccccccccccc}
    \toprule
    Datasets &\quad\; fo-MAML &\quad\; Proto-MAML &\quad\; \begin{tabular}{@{}c@{}}ALFA\\+fo-MAML\end{tabular} &\quad\; \begin{tabular}{@{}c@{}}ALFA\\+Proto-MAML\end{tabular} &\quad\; \begin{tabular}{@{}c@{}}Approximate GAP\\+fo-MAML\end{tabular} &\quad\; \begin{tabular}{@{}c@{}}Approximate GAP\\+Proto-MAML\end{tabular} &\quad\; \begin{tabular}{@{}c@{}}GAP\\+fo-MAML\end{tabular} &\quad\; \begin{tabular}{@{}c@{}}GAP\\+Proto-MAML\end{tabular}\\
    \midrule
    ImageNet &\quad\; $37.83$ &\quad\; $46.52$ &\quad\; $48.88$ &\quad\; $49.87$ &\quad\; $51.57$ &\quad\; $53.04$ &\quad\; $51.72$ &\quad\; $\mathbf{53.54}$\\
    Omniglot &\quad\; $83.92$ &\quad\; $82.69$ &\quad\; $76.67$ &\quad\; $78.40$ &\quad\; $88.13$ &\quad\; $85.82$ &\quad\; $\mathbf{88.29}$ &\quad\; $86.54$\\
    Aircraft &\quad\; $76.41$ &\quad\; $75.23$ &\quad\; $68.82$ &\quad\; $71.88$ &\quad\; $81.05$ &\quad\; $79.30$ &\quad\; $\mathbf{81.52}$ &\quad\; $78.59$\\
    Birds &\quad\; $62.43$ &\quad\; $69.88$ &\quad\; $66.32$ &\quad\; $68.57$ &\quad\; $73.85$ &\quad\; $74.06$ &\quad\; $74.16$ &\quad\; $\mathbf{74.57}$\\
    Textures &\quad\; $64.14$ &\quad\; $68.25$ &\quad\; $68.72$ &\quad\; $70.23$ &\quad\; $70.85$ &\quad\; $68.12$ &\quad\; $\mathbf{70.92}$ &\quad\; $68.73$\\
    Quick Draw &\quad\; $59.73$ &\quad\; $66.84$ &\quad\; $66.07$ &\quad\; $63.72$ &\quad\; $67.26$ &\quad\; $68.92$ &\quad\; $67.95$ &\quad\; $\mathbf{69.67}$\\
    Fungi &\quad\; $33.54$ &\quad\; $41.99$ &\quad\; $37.52$ &\quad\; $43.76$ &\quad\; $38.78$ &\quad\; $42.22$ &\quad\; $38.88$ &\quad\; $\mathbf{43.61}$\\
    VGG Flower &\quad\; $79.94$ &\quad\; $88.72$ &\quad\; $86.79$ &\quad\; $89.09$ &\quad\; $84.48$ &\quad\; $87.42$ &\quad\; $85.10$ &\quad\; $\mathbf{87.90}$\\
    \midrule
    Traffic Sign &\quad\; $42.91$ &\quad\; $52.42$ &\quad\; $65.13$ &\quad\; $58.46$ &\quad\; $67.72$ &\quad\; $68.12$ &\quad\; $68.32$ &\quad\; $\mathbf{68.73}$\\
    MSCOCO &\quad\; $29.37$ &\quad\; $41.74$ &\quad\; $43.05$ &\quad\; $46.17$ &\quad\; $43.26$ &\quad\; $45.76$ &\quad\; $43.26$ &\quad\; $\mathbf{46.04}$\\
    MNIST$^{\star}$ &\quad\; N/A &\quad\; N/A &\quad\; N/A &\quad\; N/A &\quad\; $95.74$ &\quad\; $94.73$ &\quad\; $\mathbf{95.92}$ &\quad\; $94.86$\\
    CIFAR-10$^{\star}$ &\quad\; N/A &\quad\; N/A &\quad\; N/A &\quad\; N/A &\quad\; $73.04$ &\quad\; $72.62$ &\quad\; $\mathbf{73.50}$ &\quad\; $73.25$\\
    CIFAR-100$^{\star}$ &\quad\; N/A &\quad\; N/A &\quad\; N/A &\quad\; N/A &\quad\; $60.72$ &\quad\; $59.00$ &\quad\; $\mathbf{61.29}$ &\quad\; $60.72$\\
    \midrule
    Average seen &\quad\; $62.24$ &\quad\; $67.52$ &\quad\; $64.97$ &\quad\; $66.94$ &\quad\; $69.50$ &\quad\; $69.86$ &\quad\; $69.82$ &\quad\; $\mathbf{70.39}$\\
    \midrule
    Average unseen w/o $\star$ &\quad\; $36.14$ &\quad\; $47.08$ &\quad\; $54.09$ &\quad\; $52.32$ &\quad\; $55.49$ &\quad\; $56.94$ &\quad\; $55.79$ &\quad\; $\mathbf{57.39}$\\
    Average all w/o $\star$ &\quad\; $57.02$ &\quad\; $63.43$ &\quad\; $62.80$ &\quad\; $64.02$ &\quad\; $66.70$ &\quad\; $66.77$ &\quad\; $67.01$ &\quad\; $\mathbf{67.79}$\\
    \midrule
    Average unseen w/ $\star$ &\quad\; N/A &\quad\; N/A &\quad\; N/A &\quad\; N/A &\quad\; $68.10$ &\quad\; $68.05$ &\quad\; $\mathbf{68.46}$ &\quad\; $68.39$\\
    Average all w/ $\star$ &\quad\; N/A &\quad\; N/A &\quad\; N/A &\quad\; N/A &\quad\; $68.96$ &\quad\; $69.16$ &\quad\; $69.29$ &\quad\; $\mathbf{69.62}$\\
    \bottomrule
  \end{tabular}
  }
  \label{tab:all_metadataset}
\end{table*}

\subsection{Reinforcement Learning}
\subsubsection{Datasets and Experimental setup}
For the reinforcement learning, we evaluate GAP on two benchmarks: Half-cheetah locomotion~\cite{todorov2012physics} and 2D-Navigation~\cite{finn2017model}. 
The first benchmark, half-cheetah locomotion task, aims to predict direction and velocity. 
In the goal velocity experiments, the reward is the negative absolute value between the current velocity of the agent and a goal. The goal is randomly chosen from a uniform distribution ranging between $0.0$ and $2.0$. 
In the goal direction experiments, the reward is determined based on the magnitude of the velocity in either the forward or backward direction. The direction is randomly chosen for each task. 
Each task consists of rollouts with a length of $200$, and during training, $20$ rollouts are utilized per gradient step. 
The second benchmark, 2D Navigation task, aims to enable a point agent in a 2D environment to quickly learn a policy for moving from a starting position to a goal position. 
The observation consists of the current 2D position and the actions correspond to velocity commands clipped to be in the range of $[-0.1, 0.1]$. 
A goal position is randomly selected within the unit square $[-0.5, 0.5] \times [-0.5, 0.5]$ for each task. 
The reward is calculated as the negative squared distance to the goal. 
In total, $20$ trajectories are used for one gradient update. 
For all the experiments, our model follows a neural network policy used in~\cite{finn2017model}. The details of the architecture are provided in Section~\ref{sec:5.1.2}. 
Following the experimental protocol in~\cite{finn2017model}, we employ Trust-Region Policy Optimization~(TRPO) as the meta-optimizer~\cite{schulman2015trust}, compute the Hessian-vector products for TRPO using finite differences, and utilize the standard linear feature baseline~\cite{duan2016benchmarking}. 
The feature baseline is fitted separately at each iteration for every sampled task in the batch.

\subsubsection{Results}
The results in Fig.~\ref{fig:rl} show the average reward with respect to the update steps for MAML~\cite{finn2017model}, CAVIA~\cite{zintgraf2019fast}, ModGrad~\cite{simon2020modulating}, and GAP on Half-cheetah locomotion and 2D-Navigation.
These results demonstrate that GAP significantly outperforms the state-of-the-art algorithms even after a single gradient update step. Furthermore, GAP continues to improve with additional update steps across the three benchmarks. 
In Half-cheetah direction tasks~(Fig.~\ref{fig:rl_a}), GAP achieves rewards exceeding 600 with only one step, while MAML, CAVIA, and ModGrad fall short, reaching rewards below 600. 
Additionally, for Half-cheetah velocity tasks~(Fig.~\ref{fig:rl_b}), GAP attains rewards surpassing $-80$ with a single step, whereas MAML, CAVIA, and ModGrad only reach around $-80$, $-90$, and $-100$, respectively. 
For 2D Navigation tasks, GAP consistently achieves larger rewards than MAML, CAVIA, and ModGrad with just one step, as illustrated in Fig.~\ref{fig:rl_c}. 
Across all reinforcement learning tasks, GAP achieves larger rewards than the other methods as the number of updates increases.

\begin{figure*}[!t]
\centering
\begin{subfigure}{0.3\textwidth}
    \includegraphics[width=\textwidth]{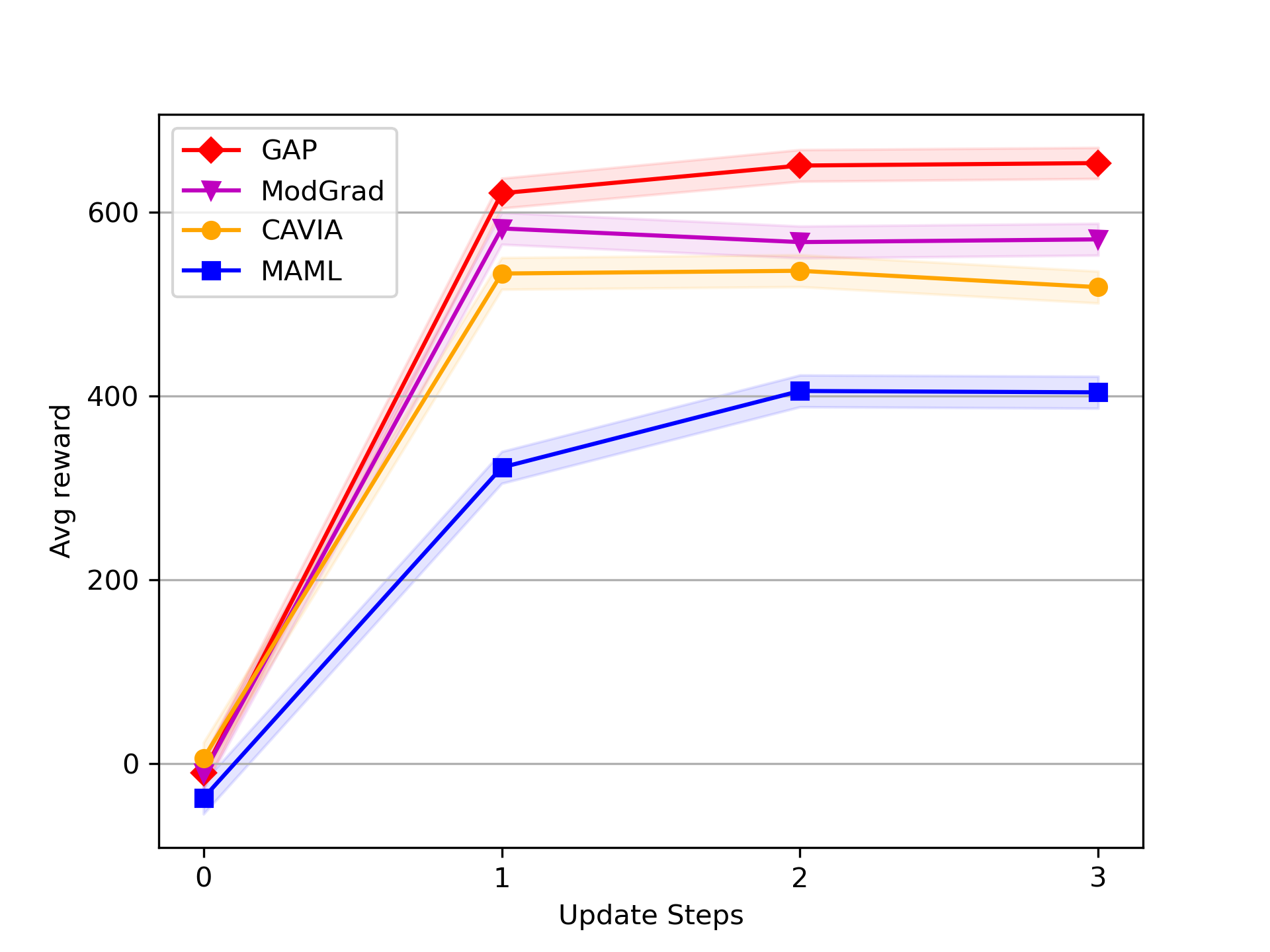}
    \caption{Half-Cheetah Direction}
    \label{fig:rl_a}
\end{subfigure}
\begin{subfigure}{0.3\textwidth}
    \includegraphics[width=\textwidth]{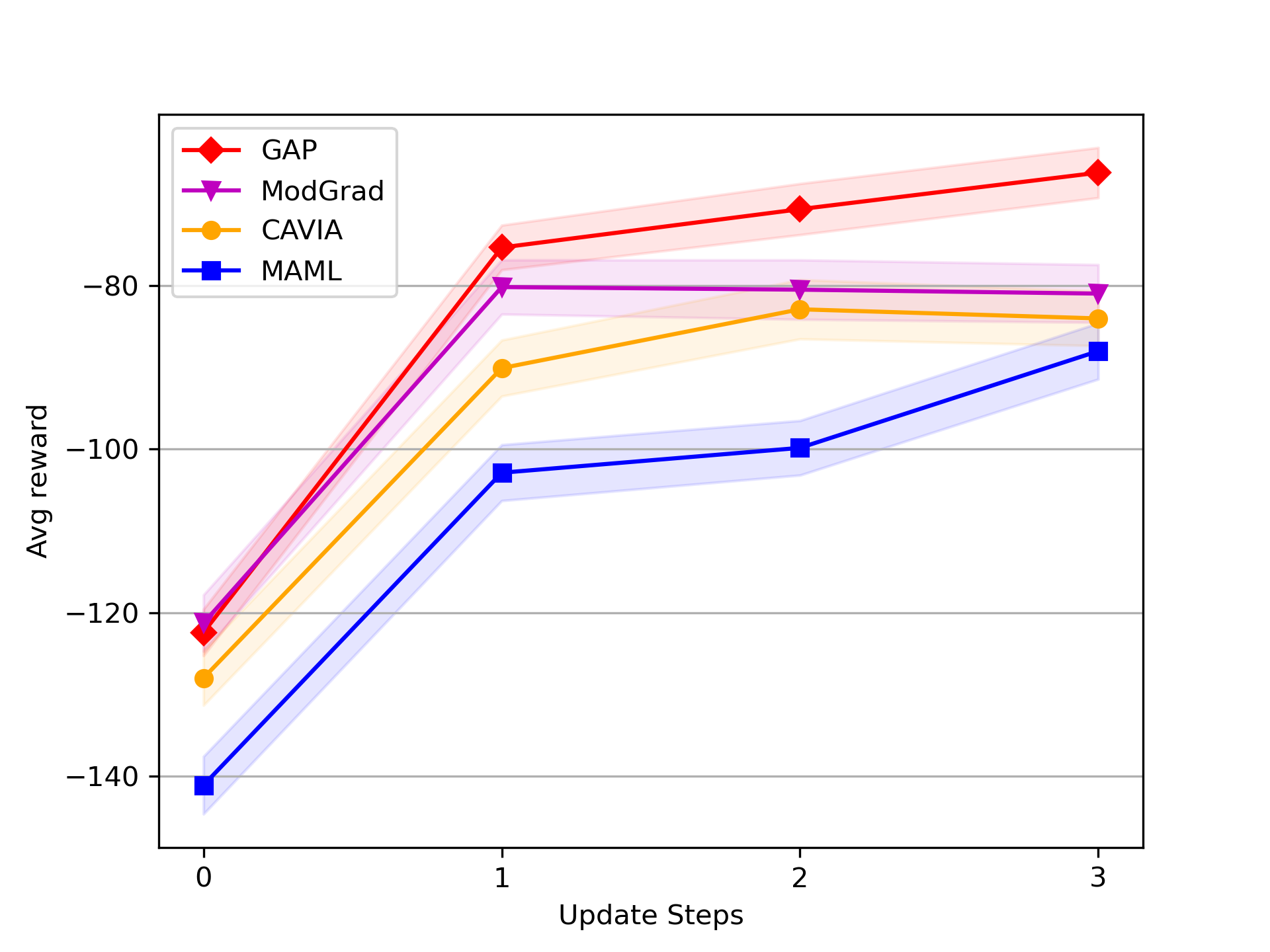}
    \caption{Half-Cheetah Velocity}
    \label{fig:rl_b}
\end{subfigure}
\begin{subfigure}{0.3\textwidth}
    \includegraphics[width=\textwidth]{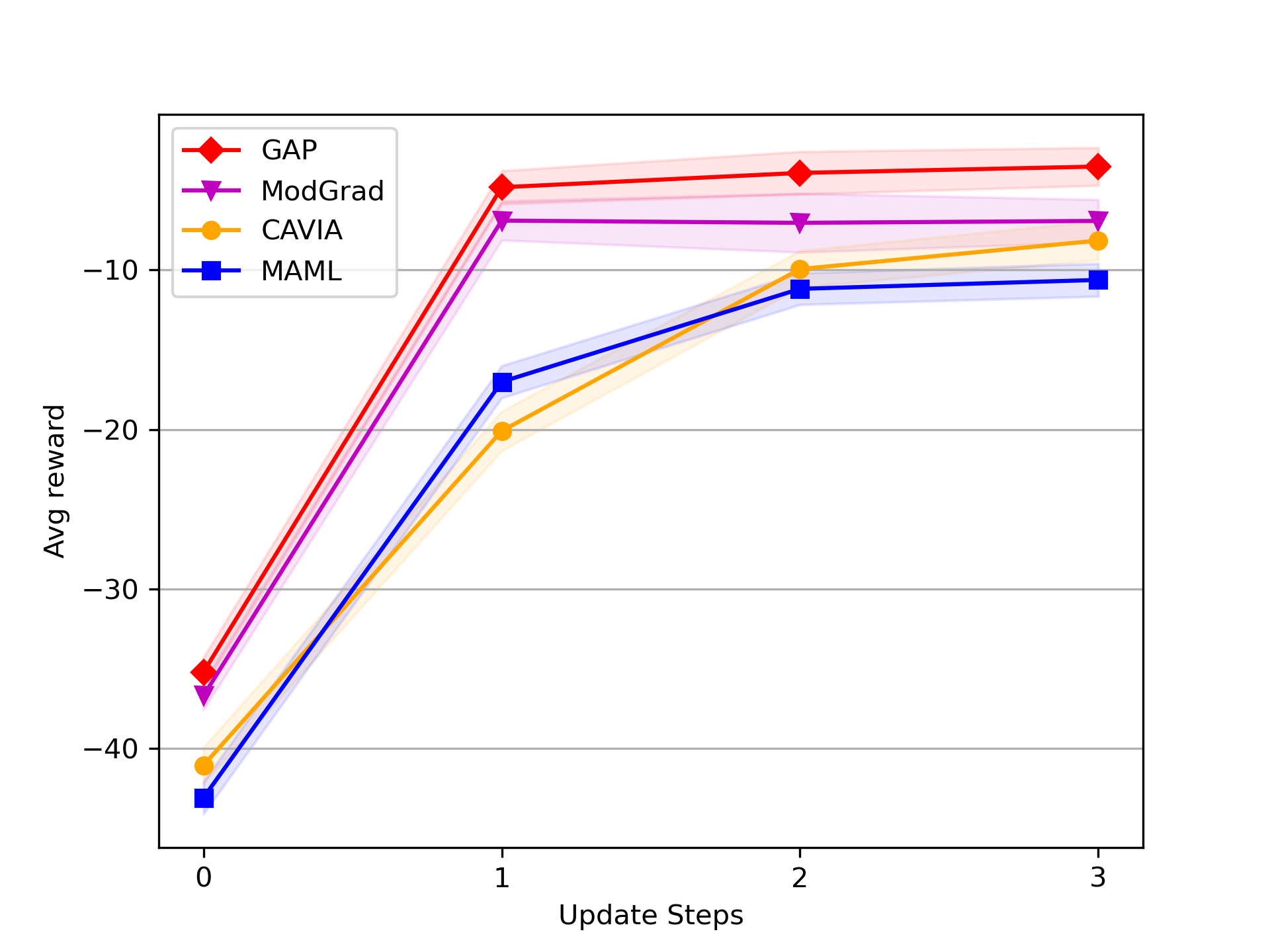}
    \caption{2D Navigation}
    \label{fig:rl_c}
\end{subfigure}
\caption{The average reward performance of MAML family and GAP models for reinforcement learning on half-cheetah direction, half-cheetah velocity, and 2D navigation. We report the performance as the number of gradient updates increases.}
\label{fig:rl}
\end{figure*}

%% file: 5_Discussion.tex
\section{Discussion}
\label{discussion1}
\subsection{Number of meta parameters}
Recent MAML family and PGD-MAML family require a large increase in the number of meta-learning parameters as shown in Table~\ref{tab:param}. One advantage of GAP is that it requires only a very small increase in the number of meta parameters, when compared to the baseline of MAML. This is possible because we transform a gradient tensor into a gradient matrix, perform the SVD of the matrix, and assign only a small number of meta parameters that correspond to the diagonal matrix of the gradient matrix. 
For the \textit{Conv-4} network, GAP requires only 0.2\% increase of 
the meta parameters. 
Although the increase in the number of meta parameters is negligible, SVD of the gradient matrix can incur a large computational burden for large networks. This is addressed by Approximate GAP. 
\begin{table}[t!]
  \caption{Comparison of the number of parameters for MAML, existing methods, and GAP.}
  \centering
  \resizebox{.50\textwidth}{!}{
  \begin{tabular}{lcr}
        \toprule
        Algorithm                &\quad\;  \# of params          &  \% increase \\
        \midrule
        MAML~\cite{finn2017model}                  &\quad\quad\; $1.2109\times10^5$  &  \\ 
        \midrule
        Meta-SGD~\cite{li2017meta}              &\quad\quad\; $2.4218\times10^5$  & 100.0\% \\
        MC~\cite{park2019meta}                    &\quad\quad\; $2.7106\times10^6$  & 2140.4\% \\
        PAMELA~\cite{rajasegaran2020meta}                &\quad\quad\; $1.6239\times10^5$  & 34.1\% \\
        MH~\cite{zhao2020meta}                    &\quad\quad\; $7.2196\times10^7$  & 59586.7\% \\
        Sparse-MAML~\cite{von2021learning}           &\quad\quad\; $2.4218\times10^5$  & 100.0\% \\
        \midrule
        GAP                  &\quad\quad\; $\mathbf{1.2131\times10^5}$  & \textbf{0.2\%} \\
        \bottomrule
  \end{tabular}
  }
  \label{tab:param}
\end{table}

\subsection{Approximate GAP vs. simple constant preconditioners}
Approximate GAP is a low-complexity method where SVD operation is avoided by approximating GAP with a constant diagonal preconditioner. A natural question to ask is how does Approximate GAP compare with other constant diagonal preconditioners. To answer this question, we have compared Approximate GAP with Meta-SGD and a modified Meta-SGD. Meta-SGD~\cite{li2017meta} is a well-known constant diagonal preconditioner~(i.e., $\text{diag}(a_1,\cdots,a_n)$) that does not need to be positive definite and we also investigate its modification with a constraint on positive definiteness. The results are shown in Table~\ref{tab:performance_summary}. It can be observed that enforcing positive definiteness can improve Meta-SGD. Furthermore, an additional improvement can be achieved by Approximate GAP. While both modified Meta-SGD and Approximate GAP are positive definite, Approximate GAP is different because it inherits an additional constraint from GAP -- a block diagonal structure where a constant diagonal matrix $M$ is repeated~(i.e., $\text{blkdiag}(M,\cdots,M)$). The inherited constraint provides a gain over the modified Meta-SGD. 

\begin{table}[t!]
\caption{Performance comparison of three constant (i.e., non-adaptive) preconditioners for 5-way 1-shot on mini-ImageNet: Meta-SGD, Meta-SGD modified to satisfy positive definiteness, and our Approximate GAP.}
\centering
\resizebox{0.5\textwidth}{!}{
    \begin{tabular}{lccc}
    \toprule
    Algorithm       & Structure    & \begin{tabular}{@{}c@{}}Riemannian metric \\ (i.e., positive definite)\end{tabular}   & Acc. (\%)    \\
    \midrule
    Meta-SGD                             & $\text{diag}(a_1,\cdot,a_n)$                 & X                   & $50.47\%$   \\
    Meta-SGD with positive definiteness  & $\text{diag}(a_1,\cdot,a_n)$                 & O                   & $52.39 \%$   \\
    \midrule
    Approximate GAP                     & $\text{blkdiag}(M,\cdots,M)$                 & O                   & $\mathbf{53.52 \%}$ \\
    \bottomrule
    \end{tabular}
}
\label{tab:performance_summary}
\end{table}

\subsection{Does GAP learn a useful preconditioner}
While a Riemannian metric can be helpful, it does not mean any Riemannian metric will result in an improvement.
For the true parameter space with a specific underlying structure, the corresponding Riemannian metric needs to be applied to enable steepest descent~\cite{amari1996neural,amari1998natural}. 
For the special case of a two-layer neural network with a mean squared error~(MSE) loss, it was proven that Fisher information matrix is the corresponding Riemannian metric~\cite{amari1998natural}. 
For a general neural network, however, a proper Riemannian metric is unknown and it needs to be learned. 
In our work, we have devised a method to guarantee a Riemannian metric and have used the outer-loop optimization to learn the Riemannian metric.
In general, the learned Riemannian metric is unlikely to correspond perfectly to the true parameter space. Then, an important question is if the Riemannian metric learned by GAP is close enough to the desired one and if it is useful. To investigate this issue, we performed an ablation study by not applying the preconditioner $\mathbf{P}_{\text{GAP}}$. After training a GAP model, we have evaluated the performance with and without applying $\mathbf{P}_{\text{GAP}}$. The results are shown in Table~\ref{tab:precon} and clearly the preconditioner learned with the outer-loop optimization plays an essential role for improving the performance.

\begin{table}[t!]
\caption{Ablation study of $\mathbf{P}_{\text{GAP}}$ on mini-ImageNet. Performance of the GAP-trained model is significantly affected by not applying $\mathbf{P}_{\text{GAP}}$.}
  \centering
  \resizebox{.5\textwidth}{!}{
  \begin{tabular}{lcc}
        \toprule
        Algorithm         &\quad\;  1-shot               & 5-shot \\
        \midrule
        GAP w/o $\mathbf{P}_{\text{GAP}}$              &\quad\quad\; $48.23 \pm 0.80$  & $65.80 \pm 0.75$ \\ 
        GAP w/  $\mathbf{P}_{\text{GAP}}$              &\quad\quad\; $54.86 \pm 0.85$  & $71.55 \pm 0.61$ \\
        \bottomrule
  \end{tabular}
  }
   \label{tab:precon}
\end{table}

\subsection{Why is preconditioner helpful for meta-learning}
When the batch size is small, the resulting empirical gradient can be noisy~\cite{zhang2019taming, simon2020modulating}. A typical few-shot learning has only a small number of samples for the inner-loop optimization, and its gradient can be noisy. On the other hand, it was shown in~\cite{amari2020does} that preconditioned gradient descent with a positive definite preconditioner can achieve a lower risk than gradient descent when the labels are noisy, the model is mis-specified, or the signal is misaligned with the features. Under a misalignment, a properly chosen positive definite preconditioner can generalize better than gradient descent~\cite{amari2020does}. 
Considering the noisy gradient of inner loop optimization, it can be surmised that a positive definite preconditioner that is adaptive (i.e., a Riemannian metric) can be helpful for improving MAML. Note that the noisy case is in contrast to the case of supervised learning with a large amount of data~\cite{amari2020does}.

%% file: 6_Conclusion.tex
\section{Conclusion}
\label{sec:conlusion}
In this work, we proposed a new preconditioned gradient descent method called GAP, which is a PGD-MAML algorithm utilizing SVD operations to achieve a better generalization. We theoretically prove that GAP's preconditioner satisfies two desirable properties: the fully adaptive property and the Riemannian metric property. Thanks to the fully adaptive property, GAP can handle both inner-step and individual tasks with a diversity within the MAML framework. Additionally, GAP can enable steepest descent on the parameter space owing to the Riemannian metric property. Furthermore, we provide an efficient approximation called Approximate GAP to alleviate computational cost problems associated with SVD operations. Through extensive experiments over a variety of few-shot learning tasks, such as few-shot regression, few-shot classification, cross-domain few-shot classification, few-shot domain generalization, and reinforcement learning, we demonstrate the effectiveness, applicability, and generalizability of GAP.

%% file: 7_Ack_bio.tex
\section*{Acknowledgement}
\label{sec:acknowledgement}
This work was supported by ETRI [23ZR1100, A Study of Hyper-Connected Thinking Internet Technology by autonomous connecting, controlling and evolving ways], NRF (NRF-2020R1A2C2007139), IITP [NO.2021-0-01343, Artificial Intelligence Graduate School Program (Seoul National University)], and the New Faculty Startup Fund from Seoul National University.